\documentclass[final,12pt]{msml2021} 

\pdfoutput=1

\usepackage{url}
\usepackage{dsfont}
\usepackage{xspace}
\usepackage[outdir=./]{epstopdf}
\usepackage{graphicx,float,pgfplots,wrapfig,sidecap,lipsum}
\usepackage{booktabs}
\usepackage{paralist}

\usepackage{algorithm}
\usepackage[noend]{algpseudocode}
\usepackage{amsfonts}
\usepackage{lipsum}
\usepackage{amsmath}
\usepackage{amssymb}
\usepackage{enumitem}
\usepackage{xcolor}
\usepackage[font=normal,labelfont=bf]{caption}
\usepackage{mathtools} 
\usepackage{bbm}
\usepackage{multirow}
\usepackage{comment}

\usepackage{tablefootnote}
\usepackage{tikz}
\usetikzlibrary{fit}
\usetikzlibrary{calc,shapes}
\usetikzlibrary{decorations.pathmorphing} 
\usetikzlibrary{fit}					
\usetikzlibrary{backgrounds}	

\usepackage{bm}
\usepackage{mathrsfs} 
\usepackage{stmaryrd}
\usepackage[utf8]{inputenc}
\usepackage[english]{babel}
\usepackage{diagbox}

\usepackage[utf8]{inputenc}
\usepackage{pgfplots}
\pgfplotsset{compat=newest}
\usepgfplotslibrary{groupplots}
\usepgfplotslibrary{dateplot}
\usepackage{soul}

\def\beq{\begin{equation}}

\def\eeq{\end{equation}\noindent}

\newcommand{\bp}{\begin{psfrags}}
\newcommand{\ep}{\end{psfrags}}
\newcommand{\bc}{\begin{center}}
\newcommand{\ec}{\end{center}}

\floatname{algorithm}{Algorithm}

\colorlet{soulcyan}{cyan!30}

\makeatletter
\let\Ginclude@graphics\@org@Ginclude@graphics
\makeatother
\newcommand{\norm}[1] {\left\lVert #1 \right\rVert}

\renewcommand{\vec}[1]{\boldsymbol{#1}}
\newcommand{\mytitle}
{Practical and Fast Momentum-Based Power Methods}
\newcommand{\myshorttitle}{Practical and Fast Momentum-Based Power Methods}
\title[\myshorttitle]{\mytitle}
\usepackage{times}

\msmlauthor{%
 \Name{Tahseen Rabbani} \Email{trabbani@cs.umd.edu}\\
 \addr Department of Computer Science \\ University of Maryland, College Park, MD
 \AND
 \Name{Apollo Jain} \Email{apollo.jain@streresearch.com}\\
 \addr Systems and Technology Research \\ Sensors Division, Arlington, VA
 \AND
 \Name{Arjun Rajkumar} \Email{rajkumar@umd.edu}\\
 \addr Department of Computer Science \\ University of Maryland, College Park, MD
 \AND
 \Name{Furong Huang} \Email{furongh@cs.umd.edu}\\
 \addr Department of Computer Science \\ University of Maryland, College Park, MD
}

\begin{document}
\maketitle

\begin{abstract}%
  The power method is a classical algorithm with broad applications in machine learning tasks, including streaming PCA, spectral clustering, and low-rank matrix approximation. The distilled purpose of the vanilla power method is to determine the largest eigenvalue (in absolute modulus) and its eigenvector of a matrix. A momentum-based scheme can be used to accelerate the power method, but achieving an optimal convergence rate with existing algorithms critically relies on additional spectral information that is unavailable at run-time, and sub-optimal initializations can result in divergence. In this paper, we provide a pair of novel momentum-based power methods, which we call the \textit{delayed momentum power method} (DMPower) and a streaming variant, the \textit{delayed momentum streaming method} (DMStream). Our methods leverage inexact deflation and are capable of achieving near-optimal convergence with far less restrictive hyperparameter requirements. We provide convergence analyses for both algorithms through the lens of perturbation theory. Further, we experimentally demonstrate that DMPower routinely outperforms the vanilla power method and that both algorithms match the convergence speed of an oracle running existing accelerated methods with perfect spectral knowledge.

\end{abstract}


\section{Introduction}\label{sec:intro}


Approximating the dominant eigenvector of a matrix $A\in\mathbb{R}^{d\times d}$ is a task common to many statistical and industrial applications. The vanilla power method is a simple and inexpensive algorithm for computing the dominant eigenvector $v_1$ of a matrix $A$. For an initial $q_0\in\mathbb{R}^d$ non-orthogonal to $v_1$, the power method performs the following update eventually converging to $v_1$,
\begin{equation}
\label{intro-vanilla-update}
q_k = A^kq_{k-1}/||A^kq_{k-1}|| . 
\end{equation}
Owing to its ease of implementation and modest assumptions for convergence, the power method has found its use in a variety of machine learning tasks. It can be used to assist the k-means algorithm for class separation of large datasets, which is referred to as power iteration clustering (PIC) \protect{\citep{lin2010power,lin2010very,thang2013deflation}}. It is also used in sparse PCA, which projects data onto sparse principal components, that is, components with small $\ell_0$ norm \protect{\citep{journee2010generalized,yuan2013truncated}}. In particular, word-embedding matrices for NLP models can be dimensionally-reduced via sparse PCA \protect{\citep{gawalt2010sparse, drikvandi2020sparse}.} 

De Sa et al. \citep{de2018accelerated} introduced the \textit{power method with momentum}, abbreviated as Power+M along with a stochastic variant, Mini-Batch Power+M. Both of these methods outperform their vanilla counterparts (the stochastic version of equation \ref{intro-vanilla-update} uses instead an unbiased estimate $\widehat{A}$ of $A$). Here, speed is measured in the sense of iteration complexity, i.e., the number of outer loop iterations/updates required to output a vector $q_k$ with precision $\epsilon$, i.e., $\sin^2 \theta(q_k, v_1)\triangleq 1-(q_k^{\top}v_1)^2<\epsilon$. Recently, other algorithms have been developed which adopt a momentum-based scheme \citep{kimstochastic,mai2019noisy}. However, a notable drawback to Power+M, Mini-Batch Power+M, and other existing momentum-based methods is that achieving accelerated iteration complexity requires knowledge of $\lambda_2$, the second greatest eigenvalue of $A$, which is an impractical assumption. Specifically, the momentum coefficient $\beta$, a hyperparameter of momentum-based power methods, must be set near $\lambda_2^2/4$ to achieve improved iteration complexity over the vanilla power method. 

In this work, we develop a scheme which enjoys a near-optimal acceleration without the strict spectral knowledge requirements of other momentum-based methods. Our scheme consists of two phases. In the \textit{pre-momentum phase}, we run a vanilla (or stochastic for the online setting) power method to approximate the dominant eigenvector and an inexact Hotelling deflation \citep{saad2011numerical} to estimate $\lambda_2$ and later assign $\beta\approx{\lambda_2^2}/4$ (our momentum coefficient). In the \textit{momentum phase}, we run Power+M (or Mini-Batch Power+M for the online setting) with the near-optimal $\beta$ assignment taken from the previous phase for the remaining iterations until convergence. Our main contributions, the \textit{delayed momentum power method} (DMPower) and the \textit{delayed momentum streaming power method} (DMStream), are realizations of this scheme.

\paragraph{Relaxing Spectral Knowledge}
 As explained above, an optimal acceleration of momentum-based methods relies on the proper selection of $\beta$. Previous approaches rely on expensive guess-and-check auto-tuning, whereby the user chooses $\beta$, and at each round conducts many experimental iterations with $0.67\beta$, 0.99$\beta$, $\beta$, 1.01$\beta$, etc.,  revising the coefficient after determining which adjustment results in the largest Rayleigh quotient \citep{de2018accelerated}. To the best of our knowledge, DMPower and DMStream are the first momentum-based algorithms to approximate optimal $\beta$ in a partially-adaptive manner while still benefiting from acceleration. 
 We present an informal version of our major results.
\begin{theorem}[\bf Informal] \label{dmpminformal}
 Let $\Delta = |\lambda_1-\lambda_2|$ denote the absolute difference between the largest and second-largest eigenvalues. With high probability, our proposed practical \emph{DMPower}, after an efficient pre-momentum warm-up stage, outputs an $\epsilon$-close estimate of the leading eigenvector within the state-of-the-art $\mathcal{O}\Bigl(\frac{1}{\sqrt{\Delta}}\log\big(\frac{1}{\epsilon}\big)\Bigr)$ iteration complexity using a momentum acceleration, without requiring knowledge of $\lambda_2$ or  hyperparameter selection for $\lambda_2$.
 \emph{DMPower} is extended to \emph{DMStream} in the streaming setting, with similar iteration complexity.
\end{theorem}

In section \ref{sec:analysis} we provide the full version the above theorem along with a companion streaming theorem. Although the momentum phases utilize existing methods, neither of our algorithms are true hybrids; DMPower and DMStream are the first of their kind to utilize inexact Hotelling deflation for second eigenvalue recovery with guarantees. We will show that the iteration complexity of the DMPower and DMStream momentum phases respectively match the iteration complexity of Power+M and Mini-Batch Power+M without having to assign $\beta$ at initialization.

Our experiments show that DMPower converges faster than the vanilla power method to various precisions of estimation and for matrices with tighter eigengaps, $\Delta=0.01, \Delta=0.001$, our DMPower matches the convergence speed of Power+M running with optimal $\beta$. DMPower also outperforms than the vanilla power method and closely mimics the performance of optimal Power+M when employed in the unsupervised learning task of spectral clustering.  Additionally, in comparison to another concurrent iteration method, the simultaneous power iteration, our experiments demonstrate that DMPower determines a more accurate approximation of $\lambda_2$. DMStream outputs a more accurate estimation of the dominant eigenvector for a variety of batch sizes when compared against Oja's algorithm and performs nearly identically to Mini-Batch Power+M initialized with optimal $\beta$. 

\paragraph{Summary of Contributions}
(1) Our proposed algorithms DMPower and DMStream achieve close to optimal performance when compared against other momentum-based power methods, which have been initialized with a priori unknowable optimal hyperparameters.
(2) To the best of our knowledge, we are to first to provide a convergence analysis of an inexact deflation receiving approximate dominant eigenvectors supplied by a power iteration, both in the deterministic setting (Lemma \ref{power-perturbation}) and streaming setting (Proposition \ref{gt-bound}).
(3) Our Proposition \ref{practical-bound} provides a guarantee for acceleration if one can provide lower bounds on $|\lambda_1-\lambda_2|$ and $|\lambda_2-\lambda_3|$, which is a far more practical requirement at run-time.
(4) For many estimation precision requirements, DMPower outperforms the state-of-the-art Lanczos algorithm in iteration complexity when factoring in the recommended number $d$ of tri-diagonalization iterations needed for numerical stability of the Lanczos algorithm. Additionally, DMPower runs noticeably faster (in seconds) than the Lanczos algorithm. 

\section{Related Works}\label{sec:related}
\paragraph{Speedy Deterministic Power Methods}
Variations of the vanilla power method intended to improve its iteration complexity of $\mathcal{O}(\frac{1}{\lambda_1-\lambda_2}\log\frac{1}{\epsilon})$ have been suggested. The Lanczos algorithm \citep{golub2012matrix} itself may be thought of as a fast variation of the power method, which has iteration complexity $\mathcal{O}(\frac{1}{\sqrt{\lambda_1-\lambda_2}}\log\frac{1}{\epsilon})$ but only after a tri-diagonalization process which is expensive in high dimensions. Lei et al. \citep{lei2016coordinate} consider a coordinate-wise update with complexity $\mathcal{O}\bigl(\frac{\lambda_1}{\lambda_1-\lambda_2}\log\frac{\tan \theta_0}{\epsilon}\bigr)$. Based on the heavy ball method first studied by Polyak \citep{polyak1964some}, De Sa et al. \citep{de2018accelerated} first proposed the addition of a momentum term to accelerate the basic power method. The term is controlled by a momentum coefficient, $\beta$, which is a hyperparameter selected at initialization. Their method achieves $\mathcal{O}\bigl(\frac{1}{\sqrt{\lambda_1-\lambda_2}}\log\frac{1}{\epsilon} \bigr)$, but only with a precise selection of $\beta$ requiring unrealistic spectral knowledge at run-time. Mai and Johannson \citep{mai2019noisy} created a momentum-based algorithm \texttt{NAPI} for solving the canonical correlation analysis (CCA) problem similar requiring impractical hyperparameter selection. In contrast, our proposed DMPower requires no spectral knowledge and achieves the same speed as existing momentum methods running optimally. 

\paragraph{Block Iterations}
To enjoy the effects of acceleration, DMPower concurrently runs a second power iteration for a finite number of rounds and extracts an optimal momentum coefficient. This approach is reminiscent, though not the same as block power iterations such as the simultaneous power iteration for matrices \citep{trefethen1997numerical}, extended to tensors by Wang and Lu \citep{wang2017tensor}, which intends to recover multiple eigenvectors at once. The simultaneous power method for matrices is well-known to suffer from rounding errors \citep{golub2012matrix}\citep{borm2012numerical}. In contrast, our inexact deflation is experimentally shown to achieve greater accuracy in extracting the second eigenvalue, which is critically important when eigengaps are tight. The improved "practical" simultaneous iteration and QR algorithm \citep{borm2012numerical} compute a QR factorization followed by a reversed RQ computation, which requires $O(d^2)$ more flops than inexact deflation. Furthermore, these block iterations do not employ any acceleration schemes: the convergence rate of an $\epsilon$-close approximation of $v_1$ in a 2-vector simultaneous iteration is $\mathcal{O}\bigl(\frac{1}{\min\{\lambda_1-\lambda_2,\lambda_2-\lambda_3\}} \log\frac{1}{\epsilon} \bigr)$.

\paragraph{Inexact Power Methods} 
In the pre-momentum phase, DMPower relies on an inexact deflation step every round. This step may be regarded as a Hardt-Price noisy power method, which has been studied in differential privacy-preserving PCA \citep{hardt2013beyond} \citep{kapralov2013differentially}, but primarily as a meta-algorithm. The DMPower may be regarded then as one practical application of the noisy power method, with an added momentum phase and thus requiring novel analysis beyond the scope of existing noisy power method literature. Furthermore, most existing results on deflation schemes operate under the assumption that eigenvectors and eigenvalues have been exactly recovered, but in practice, only approximate eigenvectors are used. To the best of our knowledge, we are the first to provide a guarantee on obtaining a second eigenvalue through deflation using inaccurate dominant eigenvectors (Lemma \ref{power-perturbation} and Proposition \ref{gt-bound}). Our algorithm is the first in the class of momentum-based power methods to incorporate an acceleration scheme without the requirement of guessing an approximation of the second eigenvalue at runtime.  

\paragraph{Streaming Methods} DMStream is a streaming companion to DMPower, which instead uses unbiased estimates of the underlying covariance matrix $A$ for all of its updates. The streaming setting assumes that $A$ is either inaccessible or expensive to obtain, and several algorithms have been developed to address this situation \citep{mitliagkas2013memory, shamir2015stochastic, kimstochastic, jain2016streaming}. The momentum phase of DMStream is built upon Mini-Batch Power+M \citep{de2018accelerated}, whose iteration complexity is $\mathcal{O}\Bigl(\frac{1}{\sqrt{\lambda_1-\lambda_2}}\log\big(\frac{1}{\epsilon}\big)\Bigr)$, which is desirable in that it matches the offline state-of-the-art Lanczos complexity, but similar to Power+M, requires a momentum hyperparameter $\beta$ close to $\lambda_2^2/4$. Just like DMPower, DMStream efficiently approximates $\lambda_2$ in its pre-momentum phase and then drops into a momentum phase where it enjoys the aforementioned optimal offline iteration complexity.  

\paragraph{Gradient Descent Methods} From the optimization perspective, leading eigenvector computation is equivalent to minimizing $-\textbf{x}^\top A \textbf{x}$ where $\textbf{x}\in S^2$, i.e., the unit sphere. Although this problem is geodesically non-convex, it is possible to use gradient descent to solve this problem \protect{\citep{absil2009optimization,wen2013feasible,pitaval2015convergence}}.In particular, the global convergence rate has been shown to be $\mathcal{O}\Bigl(\bigl(\frac{\lambda_1}{\lambda_1-\lambda_2}\bigr)^2\log\frac{1}{\epsilon}\Bigr)$, which is generally incomparable to the guarantee of DMPower \protect{\citep{xu2018convergence}}. Conjugate gradient (CG) methods have also been employed to compute the dominant eigenvector. The Fletcher-Reeves Gradient Descent (FRGD) \protect{\citep{wang2020stochastic}}, a fully-adaptive momentum-based CG method, achieves a convergence rate of $\mathcal{O}\Bigl(\frac{\sqrt{\kappa}-1}{\sqrt{\kappa}+1}\Bigr)$, where $\kappa=\lambda_1/\lambda_n$ and thus this rate is also not directly comparable to the convergence rate of DMPower and other classical power methods, since they largely depend on the first eigengap.

\section{Accelerated Momentum-Based Power Methods via Inexact Deflation}\label{sec:algo}

\paragraph{Problem Setup} 
We first outline the setting and assumptions typical of deterministic momentum-based PCA. We discuss the streaming setting in section \ref{subsec:algo-streaming}. Let $x_1,x_2,\dots, x_n\in \mathbb{R}^d$ be data points. Our goal is to recover the top eigenvector and dominant eigenvalue of the symmetric PSD covariance matrix $A=\frac{1}{n}\sum_{i=1}^n x_ix_i^{\top}\in\mathbb{R}^{d\times d}$. We assume that $A$ has eigenvalues $1\geq \lambda_1 > \lambda_2 > \lambda_3 \geq \lambda_4 \geq \cdots \geq \lambda_d \geq 0$ with associated orthonormal eigenvectors $v_1, v_2, \dots, v_d$. Unless noted otherwise, $\norm{\cdot}$ refers to the 2-norm for vectors and matrices. We let $\Delta_{1,2}:=\lambda_1-\lambda_2$ and $\Delta_{2,3}:=\lambda_2-\lambda_3$. The vanilla power method, at each round $k=1,2,\dots$ performs the following update,
\begin{align}
    q_k &= \frac{A^kq_{k-1}}{\norm{A^kq_{k-1}}} \label{vanilla_vector_update}\\
    \nu_k &= q_k^{\top}Aq_k \label{vanilla_rayleigh}
\end{align}
where $q_0$ is a random unit vector non-orthogonal to $v_1$. Under these conditions, $q_k\rightarrow v_1$ and $\nu_k\rightarrow \lambda_1$ at a geometric convergence rate, with ratio $\bigl(\frac{\lambda_2}{\lambda_1}\bigr)^2$. Here, error is measured as the sine squared of the angle $\theta(q_k,v_1)$ between our unit $q_k$ and $v_1$, that is,  $\sin^2\theta(q_k, v_1) \triangleq 1 - (q_k^{\top}v_1)^2$.
The power method with momentum uses the alternative update,
\begin{equation}
    \label{powerm_update}
    q_k=\frac{Aq_{k-1}-\beta q_{k-2}}{\norm{Aq_{k-1}-\beta q_{k-2}}}
\end{equation}
where $q_{-1} = \vec{0}$ and $\beta$ is the \textit{momentum coefficient} chosen by the user at initialization.
Sa et al. \citep{de2018accelerated} establish the following theorem and its corollary,
\begin{theorem} [Convergence of Power+M~\citep{de2018accelerated}]
\label{PowerMThm}
Given a PSD matrix $A\in\mathbb{R}^{d\times d}$ with eigenvalues $1\geq \lambda_1>\lambda_2\geq \lambda_3\dots \lambda_d\geq 0$, running with $\lambda_2 < 2\sqrt{\beta} \leq \lambda_1$ results in $q_k$ with
\begin{equation}
    \sin^2\theta(q_k, v_1)=1-(q_k^{\top}v_1)^2\leq \frac{4}{|q_0^{\top}v_1|^2} \biggl( \frac{2\sqrt{\beta}}{\lambda_1+\sqrt{\lambda_1^2-4\beta}}\biggr)^{2k}
\end{equation}
\end{theorem}
\begin{corollary} [~\citep{de2018accelerated}]
\label{PowerMIterations}
For $\epsilon\in(0,1)$ after $T=\mathcal{O}\bigl(\frac{\sqrt{\beta}}{\sqrt{\lambda_1^2-4\beta}}\log\frac{1}{\epsilon}\bigr)$ iterations, \\$\sin^2\theta(q_T, v_1)\leq \epsilon.$  
\end{corollary}
We have from these results that $\beta\in[\lambda_2^2/4,\lambda_1^2/4)$, and minimizing $\frac{\sqrt{\beta}}{\sqrt{\lambda_1-4\beta}}$ as in Corollary \ref{PowerMIterations} over this interval results in an optimal assignment $\beta=\lambda_2^2/4$.

\paragraph{Practical Considerations}
The power method with momentum is capable of achieving a faster convergence rate than the vanilla power iteration. For $\beta=\lambda_2^2/4$, we have, as in Theorem $\ref{PowerMThm}$, a geometric convergence with ratio $\Bigl(\frac{\lambda_2}{\lambda_1+\sqrt{\lambda_1^2-\lambda_2^2}}\Bigr)^2$, which is smaller and thus faster than the vanilla convergence ratio $\bigl(\frac{\lambda_2}{\lambda_1}\bigr)^2$. However, therein lies the impracticality of the algorithm: the user will generally not have knowledge of $\lambda_2$.
Guessing the momentum coefficient $2\sqrt{\beta}>\lambda_1$ can result in extremely slow convergence and in some cases, divergence, as shown in Figure~\ref{fig:divergent}. 
In fact, there is currently no known convergence guarantee for $2\sqrt{\beta}>\lambda_1$; we notice that in the setting of Theorem \ref{PowerMThm} and its associated Corollary, that such selection of $\beta$ results in an imaginary ratio, rendering the guarantee uninformative. For practical use, we should remove $\beta$ as a hyperparameter. 

Ultimately then, we must approximate $\lambda_2$. This naturally leads us to consider a concurrent/block iteration scheme that synchronously converges towards $v_1$ and $v_2$. From our approximations of $v_2$, we may obtain Rayleigh quotient estimations of $\lambda_2$. Accuracy is critical when the eigengap $\Delta_{1,2}$ is small, and since the simultaneous power iteration experiences rounding errors as shown in Figure \ref{fig:simultaneous}, it is not appropriate for use. The more accurate \textit{practical simultaneous power iteration} \citep{borm2012numerical} employs a thin QR factorization at each step, which is expensive. Our instinct, then, is to adopt a concurrent iteration which utilizes \textit{matrix deflation} to approximate $\lambda_2$. 

\captionsetup[figure]{labelfont=,}

\begin{figure}[!htbp]
\begin{minipage}[c]{0.5\textwidth}
\centering
	\resizebox{1.0\textwidth}{!}{
	\includegraphics[]{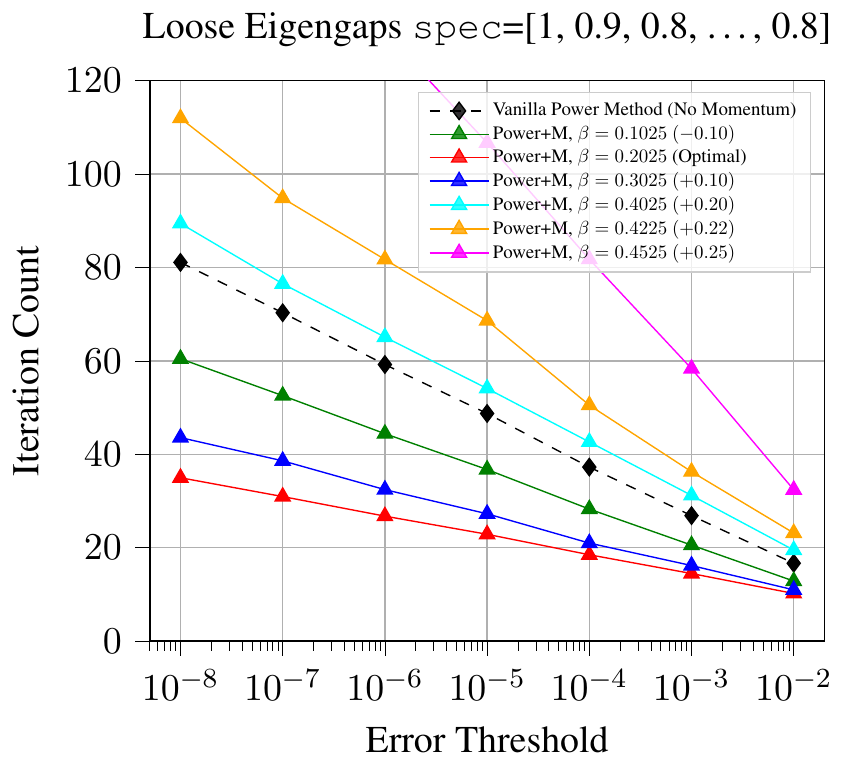}
	}
\end{minipage}\hfill
\begin{minipage}[c]{0.5\textwidth}
\caption{\textbf{Sub-optimal $\beta$ Selection for Power+M}. We measure convergence speeds at various momentum assignments. The X-axis is the error threshold $\epsilon$ between approximates $q_k,q_{k-1}$ of $v_1$ needed for Power+M to terminate, i.e., we run until $||q_k-q_{k-1}||<\epsilon$. The Y-axis is the total number of iterations $k$ needed to meet this condition according to the update equation $(\ref{powerm_update})$. Averaged over 1000 runs, each time run on a random PSD $A\in\mathbb{R}^{10\times 10}$ with spectrum $\lambda_1=1, \lambda_2=0.9, \lambda_3=0.8$, and remaining eigenvalues set to $0.8$. We observe that increased deviation from the optimal $\beta$ assignment results in worsened and eventually divergent performance.}
\label{fig:divergent}
\end{minipage}
\end{figure}

\paragraph{Approach for Smart Selection of $\beta$}
Deflation methods extract further eigenvalues along the spectrum (ordered in absolute modulus), once previous eigenvalues and eigenvectors are determined. Hotelling deflation \citep{golub2012matrix} is one such scheme upon which we model our algorithms. Assume for our symmetric PSD $A$ that it also has a positive second eigengap, i.e., $\lambda_2>\lambda_3$. If we form the deflation matrix $B = A-\lambda_1v_1v_1^{\top}$, then for $w_0$ non-orthogonal to $v_2$, we have that $w_k=\frac{Bw_{k-1}}{||Bw_{k-1}||}\rightarrow v_2$ and $q_k^{\top}Bq_k\rightarrow \lambda_2$ as $k\rightarrow\infty$. Clearly, we do not have access to $\lambda_1$ and $v_1$ (their approximation is the entire purpose of PCA), but at each round of a vanilla power iteration, we \textit{do} have approximations $\nu_k$ and $q_k$ as in update equations (\ref{vanilla_vector_update}) and (\ref{vanilla_rayleigh}), so we may instead form an inexact deflation matrix $A-\nu_kq_kq_k^{\top}$. 
\captionsetup[figure]{labelfont=,}

\begin{figure}[!thbp]
\begin{minipage}[c]{0.5\textwidth}
\centering
	\resizebox{1.0\textwidth}{!}{
	\includegraphics[]{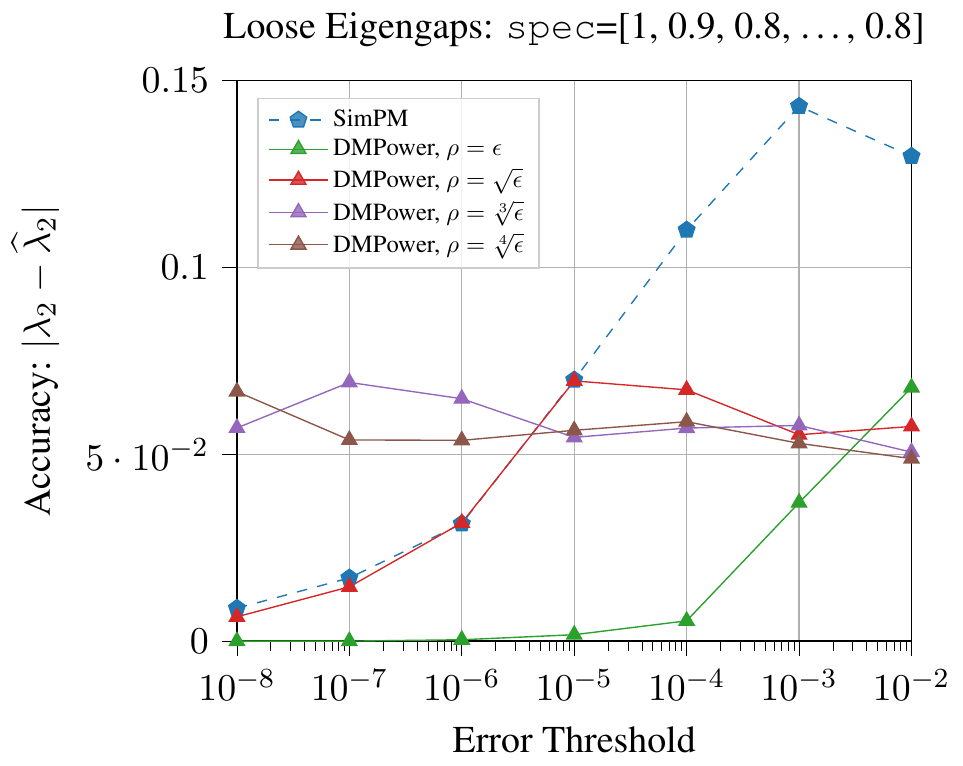}
	}
\end{minipage}\hfill
\begin{minipage}[c]{0.5\textwidth}
\caption{\textbf{Accuracy vs Simultaneous Power}. We measure the Rayleigh quotient accuracy in determining $\lambda_2$ between the Simultaneous Power Iteration (SimPM) and DMPower. The X-axis is the error threshold $\epsilon$ of the entire DMPower run.  We report the true accuracy of $\widehat{\lambda}_2$ output by the pre-momentum phase, which is reflected in the Y-axis. Averaged over 1000 runs, each time run on a random PSD $A\in\mathbb{R}^{10\times 10}$. We varied $\rho$ according to $\epsilon$ for convenience, but $\rho$ does not depend on $\epsilon$ in implementation. We observe that DMPower is overall the most accurate, while SimPM is inaccurate at low thresholds.}
\label{fig:simultaneous}
\end{minipage}
\end{figure}
One might wonder the implications of setting $\beta$ as $q_{J}^\top A q_{J}$. Setting $\beta = q_{J}^\top A q_{J}$ will result in a momentum parameter converging towards $\lambda_1$, which would not lie in the convergence zone $[\lambda_2^2/4, \lambda_1^2/4)$ of Theorem $\ref{PowerMThm}$. The next logical idea then would be to set $\beta = (q_{J}^{\top} A q_{J})^2/4$, which would converge towards $\lambda_1^2/4$, but employing this coefficient very nearly follows the same dynamics as the vanilla power method (see Theorem $\ref{PowerMThm}$) and hence, no improvement. 

Our first proposed algorithm the \textit{delayed momentum power method} (DMPower) is a realization of the above discussion and considerations, using inexact deflation to progressively approximate $\lambda_2$. Experimentally, DMPower outperforms the vanilla power method for all specified error thresholds, running near-optimally at tighter eigengaps $\Delta_{1,2}=0.01$ and $\Delta_{1,2}=0.001$. DMPower experiences decayed noise at each step, making it possible to establish a convergence guarantee, which we present in Theorem \ref{thm:main-DMPM}.

\subsection{Delayed Momentum Power Method (DMPower)}\label{subsec:algo-anchor}

\paragraph{Intuitions of DMPower} 
The delayed momentum power method, our primary contribution, experiences momentum-based acceleration after an initial waiting period with no required selection of $\beta$ at initialization. As stated before, we use inexact Hotelling deflation to approximate ${\lambda_2^2}/4$.  It is known that inexact deflation can succeed with controlled noise \citep{kapralov2013differentially,hardt2013beyond}, and we prove in Appendix \ref{app:noise-bounds} that the inexact deflation step in DMPower satisfies such conditions. We will eventually obtain an estimate $\beta\in[\lambda_2^2/4,\lambda_1^2/4)$. Using this well-behaved $\beta$, we can transition to a momentum-based update exhibiting acceleration. 

\paragraph{Overview}

\begin{algorithm}[!ht]
   \caption{Delayed Momentum Power Method (DMPower) 
   }
   \label{alg:dmpm}
\begin{algorithmic}[1]
\Require $A\in\mathbb{R}^{d\times d}$ symmetric, unit $q_0\in\mathbb{R}^d$, pre-momentum phase iterations $J$, momentum phase iterations $K$, unit $w_0\in\mathbb{R}^d$
\For{$j=1,2,\dots, J$}
\State $q_j \leftarrow Aq_{j-1}$
\State $q_j \leftarrow q_j/\norm{q_j}$
\State $\nu_j \leftarrow q_j^{\top}Aq_j$ \Comment{Rayleigh Quotient estimate of $\lambda_1$}
\State $P\leftarrow \nu_jq_jq_j^{\top}$
\State $w_j\leftarrow(A-P)w_{j-1}$ \Comment{Inexact deflation}
\State $w_j \leftarrow w_j/\norm{w_j}$
\State $\mu_j \leftarrow w_j^{\top}Aw_j$ \Comment{Rayleigh Quotient estimate of $\lambda_2$}
\EndFor
\State $\widehat{\lambda}_2=\mu_J$
\State $\beta \leftarrow \widehat{\lambda}_2^2/4$ \Comment{Approximated optimal momentum coefficient}
\State $q_1 \leftarrow q_J$  \Comment{Current estimate of $v_1$}
\State $q_0 \leftarrow \vec{0}$
\For{$k=1,2,\dots, K$} \Comment{while $\left\lVert q_{k}-q_{k-1}\right\rVert > \epsilon$}
\State $q_{k+1} \leftarrow Aq_{k}-\beta q_{k-1}$ \Comment{Momentum update}
\State $q_{k+1} \leftarrow q_{k+1}/\norm{q_{k+1}}$
\State $\nu_k \leftarrow q_{k+1}^{\top}Aq_{k+1}$
\EndFor
\Return{$q_K,\nu_K$} 
\end{algorithmic}
\end{algorithm}

Algorithm~\ref{alg:dmpm} describes DMPower. It proceeds as a routine vanilla power method, but at each step, our approximate top eigenvector and eigenvalue $q_k$ and $\nu_k$ are used to form an inexact deflation matrix $(A-P)=A-\nu_kq_kq_k^{\top}$, which is run in a power-iterative manner on an initial vector $w_0$, which is non-orthogonal to $v_2$. We refer to this portion of DMPower as the \textit{pre-momentum phase}. For a practical implementations, the first for-loop would exit once $|\mu_{k+1}-\mu_{k}|\leq \rho$, for some user-specified $\rho$, i.e., once our $\lambda_2$ approximates are close to each other. We discuss at length the practical and theoretical considerations of selecting $\rho$ in section \mbox{\ref{precision-section}}. 

After achieving $\rho$-accuracy between our approximates $\mu_k$, we set $\beta\leftarrow {\mu_k^2}/4$, and $q_0\leftarrow q_j$ as our new initial vector. We then proceed to the \textit{momentum phase}, which runs Power+M updates until $\epsilon$-accuracy is achieved among the $q_k$. Notice that $q_j$ has already made progress towards $v_1$. This greatly benefits our Power+M updates according to Theorem $\ref{PowerMThm}$, where $\sin^2\theta(q_t,v_1)$ is limited at each step by the constant $\frac{4}{|q_j^{\top}v_1|}$. We provide a convergence guarantee in Theorem~\ref{thm:main-DMPM} for DMPower.

\paragraph{Complexity}
The vanilla power iteration costs $\mathcal{O}(d^2)$ flops per round, and since we are concurrently running two vanilla power methods in addition to a Rayleigh quotient, we perform two additional $\mathcal{O}(d^2)$ flops. Although we are not asymptotically increasing the time complexity, we justify the increased flops: we prove in Theorem \ref{thm:main-DMPM} that for a fixed $\rho$, we will achieve our desired approximation of $\lambda_2$ after a finite number of rounds. Furthermore, we are only interested in the Rayleigh quotient approximations $\mu_j$ of ${\lambda_2}$ from each deflation step, and it is well-known that if $\norm{v_2-w_k}=\rho$, then $|\lambda_2-\mu_k|=\mathcal{O}(\rho^2)$ \citep{trefethen1997numerical}. That is, the Rayleigh quotient is a quadratically-accurate estimate of $v_2$.

\subsection{Delayed Momentum Streaming Power Method (DMStream)}\label{subsec:algo-streaming}
\paragraph{Overview}
We first review a typical setting for streaming PCA. Assume we have a stream of $\mathbb{R}^d$ inputs $x_1, x_2, \dots$ drawn from some unknown distribution $\mathcal{D}$ with underlying covariance matrix $A$. We wish to recover the dominant eigenvector $v_1$ of $A$.

Streaming algorithms have risen to address this challenge, which instead use a different unbiased estimate $\widehat{A}_t=\frac{1}{n}\sum_{i=1}^n x_ix_i^\top$ of $A$ at each round $t$ to conduct their updates, where $n$ is a fixed batch size, and the $x_i$'s are selected in a uniformly random manner \citep{shamir2015stochastic,jain2016streaming,mitliagkas2013memory,de2018accelerated}. In general, the total iteration complexity and runtime is dependent on several factors, including the variance $\mathbb{E}[(\widehat{A}_t-A)\otimes(\widehat{A}_t-A)]$ of each unbiased estimate (where $\otimes$ denotes the Kronecker product) and the batch size. The sample complexity is the total number of streaming inputs need overall to output an $\epsilon$-close estimate of $v_1$ for $\epsilon<1$.

Our second algorithm which we call the \textit{delayed momentum streaming power method} (DMStream) shown in Algorithm \ref{alg:dmspm} is a streaming companion to DMPower. Theorem \ref{momentum-stream-desa} due to De Sa et. al \cite{de2018accelerated} provides a convergence guarantee on a momentum-based streaming algorithm referred to as Mini-Batch Power+M. In particular, significant acceleration is experienced if $2\sqrt{\beta}\in [\lambda_2, \lambda_1)$, but as in the case of Power+M, selection of such a $\beta$ at run-time is an impractical ask. The full convergence guarantee of Mini-Batch Power+M is provided in Theorem {\ref{momentum-stream-desa}.} Similar to DMPower, DMStream approximates a converging momentum coefficient in a pre-momentum phase and then uses that coefficient to accelerate convergence in a secondary momentum phase.

\begin{algorithm}[!ht]
   \caption{Streaming Delayed Momentum Power Method (DMStream) 
   }
\label{alg:dmspm}
\begin{algorithmic}[1]
\Require Streaming samples $x_1, x_2, \dots, x_l \in\mathbb{R}^{d}$, batch size $n$, unit $q_0\in\mathbb{R}^d$, pre-momentum phase iterations $J$, momentum phase iterations $K$, unit $w_0\in\mathbb{R}^d$
\For{$j=1,2,\dots, J$}
\State Generate unbiased estimate $\widehat{A}_j=\frac{1}{n}\sum_{i=(j-1)n+1}^{jn} x_{i}x_{i}^\top $
\State $q_j \leftarrow \widehat{A}_jq_{j-1}$
\State $q_j \leftarrow q_j/\norm{q_j}$
\State $\nu_j \leftarrow q_j^{\top}\widehat{A}_jq_j$ \Comment{Rayleigh Quotient estimate of $\lambda_1$}
\State $P\leftarrow \nu_jq_jq_j^{\top}$
\State $w_j\leftarrow(\widehat{A}_j-P)w_{j-1}$ \Comment{Inexact deflation}
\State $w_j \leftarrow w_j/\norm{w_j}$
\State $\mu_j \leftarrow w_j^{\top}\widehat{A}_jw_j$ \Comment{Rayleigh Quotient estimate of $\lambda_2$}
\EndFor
\State $\hat{\lambda}_2 \leftarrow \mu_J$
\State $\beta \leftarrow \mu_J^2/4$ \Comment{Approximated optimal momentum coefficient}
\State $q_1 \leftarrow q_J$  \Comment{Current estimate of $v_1$}
\State $q_0 \leftarrow \vec{0}$
\For{$k=1,2,\dots, K$} \Comment{while $\left\lVert q_{k}-q_{k-1}\right\rVert > \epsilon$}
\State Generate unbiased estimate $\widehat{A}_k=\frac{1}{n}\sum_{i=(k-1)n+1}^{kn} x_{i}x_{i}^\top $
\State $q_{k+1} \leftarrow \widehat{A}_kq_{k}-\beta q_{k-1}$ \Comment{Momentum update}
\State $q_{k+1} \leftarrow q_{k+1}/\norm{q_{k+1}}$
\State $\nu_k \leftarrow q_{k+1}^{\top}\widehat{A}_kq_{k+1}$
\EndFor
\Return{$q_K,\nu_K$} 
\end{algorithmic}
\end{algorithm}

DMStream superficially resembles DMPower, instead using unbiased estimates for its updates. However, due to the noise introduced by estimation error of $A$ by $\widehat{A}_t$ in conjunction with the noise introduced by our imperfect estimations of $v_1$ by $q_t$, the inexact deflation step is more challenging to analyze and results in a distinct guarantee, which we provide in Theorem \ref{thm:main-DMSPM}. 

\paragraph{Complexity} 
 Each matrix-vector multiplication cost $\mathcal{O}\bigl({d^2}\big)$ with three such multiplications in every round of the pre-momentum phase (power iteration, Hotelling iteration, and an inexact Rayleigh quotient). Akin to DMPower, we justify these increased FLOPS by noting that the pre-momentum phase will terminate in a finite number of rounds, which is shown in Theorem \ref{thm:main-DMSPM}. We empirically observe that even a rough selection of $\beta$ provides us with noticeable acceleration, resulting in lower iteration complexity overall, and thus a decreased total runtime when compared to conventional streaming PCA as in Algorithm \ref{alg:streamPCA}.

\section{Convergence Analysis}\label{sec:analysis}

We now provide our two major convergence theorems. We adopt the same notations as in Algorithm \ref{alg:dmpm} and Algorithm \ref{alg:dmspm}. Both algorithms are divided into $J$ steps of a pre-momentum phase and $K$ steps of a momentum phase.
As a reminder, $\mu_j$ is the Rayleigh quotient approximation of $\lambda_2$ at step $j$ and $\beta := {\widehat{\lambda}}^2/4 = \mu_J^2/4$. For notational convenience, we let $\theta_0:=\arccos{|q_0^\top v_1|}$. 

\subsection{Delayed Momentum Power Method (DMPower)}\label{subsec:analysis-algo-anchor}
\begin{theorem}[Convergence of DMPower]\label{thm:main-DMPM}
Let $J$ represent the number of steps in the pre-momentum phase and $K$ the number of steps in the momentum phase as in Algorithm \ref{alg:dmpm}. Let $\epsilon<1$ represent the desired error threshold of our $v_1$ estimates, i.e., $\sin^2\theta(q_t,v_1)<\epsilon$ and $\rho<\min\{1/2,\sqrt{\frac{\lambda_1-\lambda_2}{\lambda_2-\lambda_d}}\}$ represent the desired error threshold of our $\lambda_2$ estimates, i.e., $|\mu_k-\lambda_2|<\rho$. Further fix $\tau>1$ and $\delta=\min\{\rho,\frac{1}{\tau \sqrt{d}}\}$. Then after
\begin{align}
J&=\mathcal{O}\bigl(\frac{1}{\lambda_1-\lambda_2}\log\frac{\tan^2\theta_0}{\delta(\lambda_2-\lambda_3)}+\frac{\lambda_2}{\lambda_2-\lambda_3}\log\frac{d\tau}{\rho}\bigr),\\ K&=\mathcal{O}\bigl(\frac{\beta}{\sqrt{\lambda_1^2-4\beta^2}}\log\frac{1}{\epsilon}\bigr) 
\end{align}
pre-momentum and momentum steps, respectively, where $\beta = {\widehat{\lambda}_2}^2/4 = \mu_J^2/4$, with all but $\tau^{-\Omega(1)}+e^{-\Omega(d)}$ probability, \emph{DMPower} outputs a vector $q_{K}$ with $\sin^2 \theta(q_K,v_1) <\epsilon$. 

\end{theorem}
\noindent \textit{Remark 1:} Our step count for Power+M convergence implicitly assumes $\lambda_2\leq 2\sqrt{\beta_J}$, i.e., our $\beta$ approximation lives on the right of $\lambda_2^2/4$. It is possible for $2\sqrt{\beta_J}\leq\lambda_2$, in which case we will still appreciate the effects of acceleration (see Theorem \ref{PowerMThm-Restated}).

\noindent \textit{Remark 2:} We require $J=\mathcal{O}\bigl(\frac{1}{\lambda_1-\lambda_2}\log\frac{\tan^2\theta_0}{\delta(\lambda_2-\lambda_3)}+\frac{\lambda_2}{\lambda_2-\lambda_3}\log\frac{d\tau}{\rho}\bigr)$ to achieve $\rho$-accuracy and a further $K=\mathcal{O}\bigl(\frac{\beta_J}{\sqrt{\lambda_1^2-4\nu_J^2}}\log\frac{1}{\epsilon}\bigr)$ steps to achieve $\epsilon$-accuracy. In the momentum phase, we absorb $q_J$ as our initial vector, which has already made convergent progress towards $v_1$. In practice, we will not need all $J+K$ iterations.

\paragraph{Proof Sketch.} We divide DMPower by its pre-momentum and momentum phases. The full proof is deferred to Appendix~\ref{app:proof-dmpm}.

\emph{Pre-momentum phase:} We regard the inexact deflation step as an exact deflation experiencing a perturbation every round. In Lemma \ref{power-perturbation}, we show that this noise decays at every step and after $J_1=\mathcal{O}\bigl(\frac{1}{\lambda_1-\lambda_2}\log\frac{\tan^2 \theta_0}{\delta(\lambda_2-\lambda_3)}\bigr)$ steps, we achieve the Hardt-Price bounds \citep{hardt2014noisy} necessary for convergence of a noisy power method. The convergence rate for noisy power methods \citep[\textrm{Corollary 1.1}]{hardt2014noisy} indicates that after an additional $J_2=\mathcal{O}\bigl(\frac{\lambda_2}{\lambda_2-\lambda_3}\log\frac{d\tau}{\rho} \bigr)$ steps, we will reach our desired $\rho$-accuracy, so in total, we need $J=J_1+J_2$ iterations to complete the pre-momentum phase. We set $\beta=\frac{\mu_J^2}{4}$ and proceed to the momentum phase.

\emph{Momentum phase:} Now that our momentum $\widehat{\lambda}_2=\mu_J^2/4$ coefficient is within the interval $[\lambda_2^2/4,\lambda_1^2/4)$, we may invoke Sa et al.'s Power+M convergence Theorem \ref{thm:main-DMPM}, which states that after $K=\mathcal{O}\Bigl( \frac{\beta_J}{\sqrt{\lambda_1^2-4\beta_J^2}}\log\bigl(\frac{1}{\epsilon}\bigr)\Bigr)$ steps of iteration on $q_J$ (which we take to be our initial vector for Power+M) steps, we will have that $\sin^2\theta(q_{J+K},v_1)<\epsilon$.

\subsection{Delayed Momentum Streaming Power Method (DMStream)}\label{subsec:DMStream-Convergence}
\begin{theorem} [Convergence of DMStream]
\label{thm:main-DMSPM}
Let $\Sigma=\mathbb{E}[(\widehat{A}_t-A)\otimes(\widehat{A}_t-A)]$, where $\widehat{A}_j=\frac{1}{n}\sum_{i=1}^n x_ix_i^\top$ represents any unbiased estimate of $A$ in DMStream with fixed batch size $n$ and $\otimes$ denotes the Kronecker product. Assume we initialize with unit $q_0\in\mathbb{R}^d$ where $d \gg 0$ and $|v_1^\top q_0| \geq 1/2$. Let $\theta_0=\arccos{|q_0^\top v_1|}$. For any $\delta <1$, $\epsilon<1$, suppose
\begin{equation}
\label{variance-cond}
     ||\Sigma||\leq \frac{(\lambda_1^2-4\beta)\delta\epsilon}{256\sqrt{d}J} =
    \frac{(\lambda_1^2-4\beta)^{3/2}\delta\epsilon}{256\sqrt{d}\sqrt{\beta}}\log^{-1}\Bigl( \frac{32}{\delta\epsilon} \Bigr),
\end{equation}
where $J$ is the total number of pre-momentum steps we have fixed at runtime. Furthermore, we let $\rho<\min\{1/2,\sqrt{\frac{\lambda_1-\lambda_2}{\lambda_2-\lambda_d}}\}$ represent the error threshold of our $\lambda_2$ estimates, i.e., $|\mu_k-\lambda_2|<\rho$. Lastly, fix $\tau>1$ and $\delta=\min\{\rho,\frac{1}{\tau \sqrt{d}}\}$.
If our batch size $n$ is chosen such that 
\begin{equation}
\label{batch-size}
\frac{n}{\log^4 n}=\mathcal{O}\Bigl(\frac{1/\gamma^2\log d}{(\lambda_2-\lambda_3)^2 d} \Bigr).    
\end{equation}
where $\gamma = \frac{\rho(\lambda_2-\lambda_3)}{10\tau\sqrt{d}}$, then after 
\begin{align}
 J &=\mathcal{O}\Biggl(\frac{1}{\lambda_1-\lambda_2}\log \biggl(\frac{\tan^2 \theta_0\tau\sqrt{d}}{\rho(\lambda_2-\lambda_3)}\biggr) + \frac{\lambda_2}{\lambda_2-\lambda_3}\log\frac{d\tau}{\rho}\Biggr),\\
 K &=\frac{\sqrt{\beta}}{\sqrt{\lambda_1-4\beta}}\log\Bigl(\frac{32}{\delta\epsilon}\Bigr)
\end{align}
 pre-momentum steps and momentum steps, respectively, with $(1-\frac{1}{n^2})(1-2\delta)(1-\tau^{-\Omega(1)}+e^{-\Omega(d)})$ probability DMStream outputs a vector $q_{K}$ such that $\sin^2\angle(q_{K},v_1)<\epsilon$.
\end{theorem}
\textit{Remark 1:} Both phases have a probabilistic guarantee, whereas in DMPower, the momentum phase was deterministically guaranteed. The probabilistic parameter for the pre-momentum phase $\tau$ while it is $\delta$ for the momentum phase. 

\noindent \textit{Remark 2:} We assume a variance condition on our estimates in equation \ref{variance-cond}. While there are many sophisticated methods designed to reduce variance \citep{shamir2015stochastic,partridge1998fast,de2018accelerated} via introduction of step sizes and anchor iterates, we do not explore these options in this paper. However, a simple strategy for reducing variance is to increase batch size $n$, since we have the relation $||\Sigma||\leq\frac{\sigma^2}{s}$, where $\sigma^2$ is the variance of a single random sample.

\noindent \textit{Remark 3:} The total sample complexity is $n(J+K)$.

\paragraph{Proof Sketch.} We divide DMStream by its pre-momentum and momentum phases. The proof is deferred to Appendix~\ref{app:proof-dm-spm}.
\emph{Pre-momentum phase:} There are two sources of noise in every round of this phase: $H_t$ the estimation error associated to $(A-\widehat{A}_j)w_j$, and $G_t$ the estimation error associated to $(\lambda_1v_1v_1^\top-\nu_jq_jq_j^\top)w_j$. Proposition \ref{dmspm-bounds} shows us how to control $||H_t||$ through batch size and Proposition \ref{gt-bound} details $||G_t||$. We require 
$J_1:=\mathcal{O}\biggl(\frac{1}{\lambda_1-\lambda_2}\log \biggl(\frac{\tan^2 \theta_0\tau\sqrt{d}}{\rho(\lambda_2-\lambda_3)}\biggr)\biggr)$ to achieve the Hardt-Price bounds and a further $J_2 :=\mathcal{O}\bigl(\frac{\lambda_2}{\lambda_2-\lambda_3}\log\frac{d\tau}{\rho} \bigr)$ to acquire the appropriate $\beta$, for a total of $J=J_1+J_2$ pre-momentum rounds. 

\emph{Momentum phase:} Now that our momentum coefficient $\beta=\widehat{\lambda}_2=\mu_J^2/4$ coefficient is within the interval $[\lambda_2^2/4,\lambda_1^2/4)$, we may invoke Sa et al.'s streaming power convergence Theorem \ref{momentum-stream-desa}, to conclude that we need $K=\frac{\sqrt{\beta}}{\sqrt{\lambda_1-4\beta}}\log\bigl(\frac{32}{\delta\epsilon}\bigr)$ steps to complete the momentum phase. 

\subsection{Precision of Inexact Deflation}
\label{precision-section}
In this section we refer to variables and notations as listed in Algorithm \ref{alg:dmpm} and Algorithm \ref{alg:dmspm}. We let $\Delta_{1,2}:=|\lambda_1-\lambda_2|$. Although one sets the number of iterations $J$ for the pre-momentum phase at run-time, we are also interested in the accuracy of our $\lambda_2$ estimates, that is, $|\lambda_2-\mu_j|$, since this will determine how quickly our momentum phase converges. In Theorem $\ref{thm:main-DMPM}$ and Theorem $\ref{thm:main-DMSPM}$, we discuss how many pre-momentum iterations $J$ are needed to ensure $|\mu_j - \mu_{j-1}|\leq \rho$, where $\rho$ is an error threshold which controls the accuracy of our final estimate $\widehat{\lambda}_2$. As such, for the remainder of this section we will look at a modification of DMPower and DMStream where $\rho$ is provided as a hyperparameter and how it affects our overall convergence.

\paragraph{Effects of inaccurate approximation of $\lambda_2$}
In Theorem \ref{thm:main-DMPM} and Theorem \ref{thm:main-DMSPM}, the overall convergence rate is dependent on our pre-momentum phase error $|\lambda_2-\widehat{\lambda}_2|<\rho$, where $\widehat{\lambda}_2=\mu_J$. Ultimately, we need our approximated momentum coefficient $\beta=\widehat{\lambda}_2^2/4 \in [\lambda_2^2/4, \lambda_1^2/4)$. In fact, even if $\beta< \lambda_2^2/4$, the momentum phase will still converge, and will still experience similar momentum effects as long as $|\beta-\lambda_2^2/4|< |\lambda_1^2/4-\lambda_2^2/4|$, which is a generalization of Theorem \ref{PowerMThm}, provided in Theorem \ref{PowerMThm-Restated}. 
We establish a proposition suggesting how accurately $\widehat{\lambda}_2$ must approximate $\lambda_2$:
\begin{proposition}
\label{prop-rhoprecision}
The momentum phase of DMPower set with momentum coefficient $\beta=\mu_J^2/4$ converges if and only if $|\lambda_2-\widehat{\lambda}_2| \leq \Delta_{1,2}.$
\end{proposition}
The proof is deferred to Appendix \ref{app:precision_bounds}. In lieu of fixed number of iterations for the pre-momentum phase, one may instead choose to exit the first for-loop if $|\mu_j-\mu_{j-1}|< \rho$, where we have now adopted $\rho$ as a hyperparameter, which is a far less aggressive option than the guesswork required with randomly selecting a convergent $\beta$ as in Power+M. We assume this termination condition for the remainder of our discussion.

By Theorem $\ref{thm:main-DMPM}$, the $\mu_j \rightarrow \lambda_2$, so we argue it is fair to assume $|\mu_j-\mu_{j-1}|\approx |\mu_j-\lambda_2|$. In which case, by the triangle inequality and Proposition \ref{precision-section}, we have that $|\mu_j-\mu_{j-1}|< \rho$ if and only if $\rho \lesssim  \frac{1}{2}\Delta_{1,2}$. Loosening $\rho$ beyond this bound does not necessarily result in divergence, however, which we experimentally observe and discuss in the next section. Current state-of-the-art bounds do not say anything meaningful about how quickly we can expect to converge/diverge outside of this bound. According to Theorem \ref{thm:main-DMPM}, the looser we set $\rho$, the fewer iterations we can expect to run in the pre-momentum phase. However, the tighter we set $\rho$, the closer our $\beta$ approaches ${\lambda_2^2}/4$, which is the optimal assignment for the momentum phase.

\paragraph{Practical selection of $\rho$}
Instead of setting $\rho$ to an exceedingly small value close to machine precision, we experimentally demonstrate in Figure \ref{fig:iteration} that DMPower is successful in the setting where $\Delta_{1,2}=\Delta_{2,3}$ for a variety of $\rho$ selections. In Figure \ref{fig:iteration} we have set $\rho=\sqrt[k]{\epsilon}$ for $k=1,2,3,4$ to demonstrate flexibility, but they are independent precision bounds; $\rho$ depends only on $\Delta_{1,2}$ by Proposition \ref{prop-rhoprecision}, not on $\epsilon$. We further observe that our DMPower outperforms the vanilla power method at nearly all error thresholds, and converges at a rate similar to optimal Power+M for tighter eigengaps according to Figure~\ref{fig:iteration}. In Figure \ref{fig:simultaneous} and Table \ref{table:two}, we demonstrate that setting $\rho$ looser causes us to non-negligibly lose precision in our approximation of $\lambda_2$. Setting $\rho=\epsilon$ tighter increases our accuracy, but our iteration complexity worsens by a noticeable amount, especially in the medium eigengap setting.

If one is certain about lower bounds $\alpha_1, \alpha_2$ of $\Delta_{1,2}$ and $\Delta_{2,3}$, respectively, then we have the following result for DMPower: 
\begin{proposition}
\label{practical-bound}
Assume $\alpha_1 \leq \Delta_{1,2}$ and $\alpha_2\leq \Delta_{2,3}$. Fix $\rho < \min\{1/2,\sqrt{\alpha_1}\}$, $\tau>1$ and let $\theta_0=\arccos{|q_0^\top v_1|}$, $\delta=\min\{\rho, \frac{1}{\tau\sqrt{d}}\}$. Then after
\begin{equation}
J=\mathcal{O}\bigl(\frac{1}{\alpha_1}\log\frac{\tan^2\theta_0}{\delta\alpha_2}+\frac{1}{\alpha_2}\log\frac{d\tau}{\rho}\bigr)
\end{equation}
pre-momentum phase steps, we output a vector $w_J$ such that if $\mu_J=w_JAw_J^\top$, then $|\lambda_2-\mu_J|<\rho^2$. 
Since $\rho^2 < \alpha_1$, our momentum phase will converge with all but $\tau^{-\Omega(1)}+e^{-\Omega(d)}$ probability. 
\end{proposition}
\begin{proof}
The size of $J$ along with the probabilistic guarantee is a simple corollary of Theorem \ref{thm:main-DMPM} in conjunction with Lemma \ref{vanillarayleighbound} to relate $\sin(w_J,v_2)$ to $|\lambda_2-\mu_J|$.
\end{proof}
It is important to note the usefulness and practicality of Proposition \ref{practical-bound}: in other momentum-based methods, having a lower bound on $\Delta_{1,2}$ was not sufficient to guarantee convergence -- one would still need the actual location of $[\lambda_2, \lambda_1)$ over [0,1]. 

\section{Experiments}\label{sec:exp}


In this section, we discuss the set of experiments we conducted to measure the performance of DMPower and DMStream against a variety of common baseline algorithms. All the experiments were run on an Intel(R) Xeon(R) E5-1650 v4 machine with 32GiB RAM and Linux OS. The algorithms were implemented in Python using the Numpy and SciPy libraries. We refer the reader Appendix E for further experimental design and analyses, including the construction of the synthetic matrices and runtime (wall-time) comparisons.
\subsection{DMPower Experiments}
\paragraph{Experimental Setup}
In these experiments, we compare DMPower against the vanilla power method, Power+M with optimal assignment of $\beta$, and the Lanczos algorithm. In each experiment, we generate a random symmetric PSD $A\in\mathbb{R}^{100 \times 100}$ with a fixed spectrum and conduct PCA using the methods listed above. In Figures \ref{fig:divergent} and \ref{fig:iteration}, we record the number of iterations that are required to achieve the desired error-tolerance $\epsilon$ between our dominant eigenvalue approximates and then take an average over 1000 runs, each time generating a new symmetric PSD $A$ with a specified spectrum. The initial vector $q_0$ is uniformly set across all methods for every run. Since we know the spectrum of these synthetic matrices, we can initialize Power+M to run with $\beta=\lambda_2^2/4$, the optimal momentum coefficient. In Figure \ref{fig:simultaneous}, we measure the accuracy of Rayleigh quotient estimates to the Simultaneous Power Iteration rather than convergence speeds.

\paragraph{Iteration Complexity Results} As shown in Figure \ref{fig:iteration}, DMPower outperforms the vanilla power method at nearly all $\epsilon$ thresholds, across a variety of $\rho$ settings. We set $\rho=\epsilon^{1/k}$ for $k=1,2,3,4$, but we must stress that $\rho$ and $\epsilon$ are independent precision bounds; $\rho$ depends on $\Delta_{1,2}$ by Proposition \ref{prop-rhoprecision}. Furthermore, at medium and tight eigengaps, observe DMPower performs nearly as well as optimal Power+M. When factoring in tridiagonalization iterations, DMPower outperforms the Lanczos algorithm, especially at tighter precisions (the Lanczos is well-known to suffer from numerical instability without corrective measures such as re-orthogonalization \citep{saad2011numerical}). See Table \ref{table:three} for a full set of raw data. We also measured the proportion of pre-momentum versus momentum phase iterations in Table \mbox{\ref{table:six}}, noting that for a variety of $\rho$ settings, most iterations of DMPower are spent in the momentum phase.
\captionsetup[figure]{labelfont=,}

\begin{figure}[!htbp]
\floatconts
{fig:iteration}
{\caption{
\textbf{Iteration Complexity Comparisons}. We compare DMPower (with different $\rho$ settings) against the vanilla power method, Power+M with optimal $\beta$, and the Lanczos algorithm. The X-axis corresponds to $\epsilon$, and the Y-axis measures performance by iteration count. Our algorithms demonstrate consistently favorable performance against the vanilla power method baseline and match optimal convergence speeds established by Power+M at medium and tight eigengap settings. \textit{Note 1}: For the Lanczos algorithm, we start by taking 100 tri-diagonalization iterations, which is the recommended number for numerical stability. \textit{Note 2:} We vary $\rho$ according to $\epsilon$ for convenience, but they are independent precision thresholds. As stated in Proposition \ref{prop-rhoprecision}, $\rho$ is dependent on $\Delta_{1,2}$. }}
{
\subfigure[$\Delta_{1,2}=\Delta_{2,3}=0.1$]{
\resizebox{0.3\textwidth}{!}{\includegraphics[]{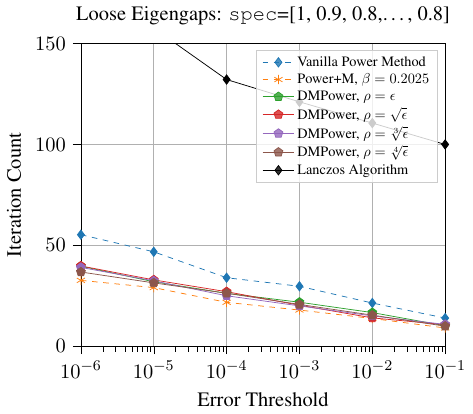}}
		\label{fig:iteration-a}
}
\hfill
\subfigure[$\Delta_{1,2}=\Delta_{2,3}=0.01$]{
\resizebox{0.3\textwidth}{!}{\includegraphics[]{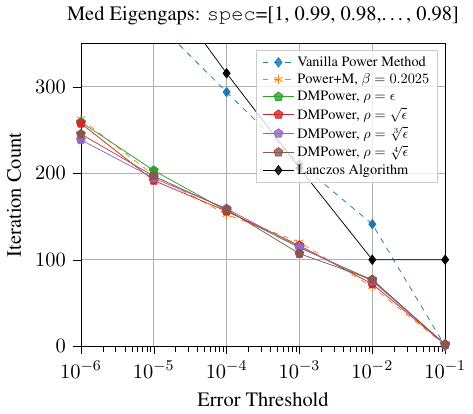}}
\label{fig:iteration-b}
}
\hfill
\subfigure[$\Delta_{1,2}=\Delta_{2,3}=0.001$]{
\resizebox{0.315\textwidth}{!}{\includegraphics[]{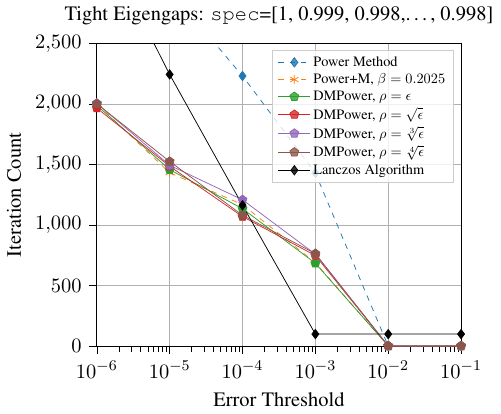}}
\label{fig:iteration-c}
}
}
\end{figure}

\paragraph{Wall-time Performance}
At loose and medium eigengaps, Power+M (optimal) registers the lowest wall-time, owing to fewer matrix-vector computations. We remind the reader that this setting is only a baseline and not practically achievable. Of particular note is that at nearly all settings, DMPower runs noticeably faster than the Lanczos algorithm. We refer the reader to Table \ref{table:seven} for full data.

\subsection{DMStream Experiments}
\paragraph{Experimental Setup}
We used a 50,000 sample subset of the MNIST dataset \citep{lecun1998gradient}, which is represented as a matrix of size $50000 \times 784$. The dataset was first pre-processed by centering and dividing the entire matrix by $\sigma\sqrt{784}$. We compared DMStream against Oja's algorithm \citep{oja1982simplified} with varying step sizes, stochastic power iteration (Algorithm \ref{alg:streamPCA}), and Mini-Batch Power+M set with optimal $\beta=\lambda_2^2/4$, where $\lambda_2$ is the second eigenvalue of the (processed) covariance matrix of MNIST. We measured performance by a commonly-used metric $\log_{10}\bigl(1 - \frac{||X^\top q_K||}{||X^\top v_1||} \bigr)$. We tested over a variety of batch sizes with $\rho=0.1$ for DMStream. For each batch size, we ran 50 iterations and averaged the results over 10 runs.

\captionsetup[figure]{labelfont=,}

\begin{figure}[!htbp]
\begin{minipage}[c]{0.50\textwidth}
\centering
	\resizebox{.8\textwidth}{!}{
	\includegraphics[]{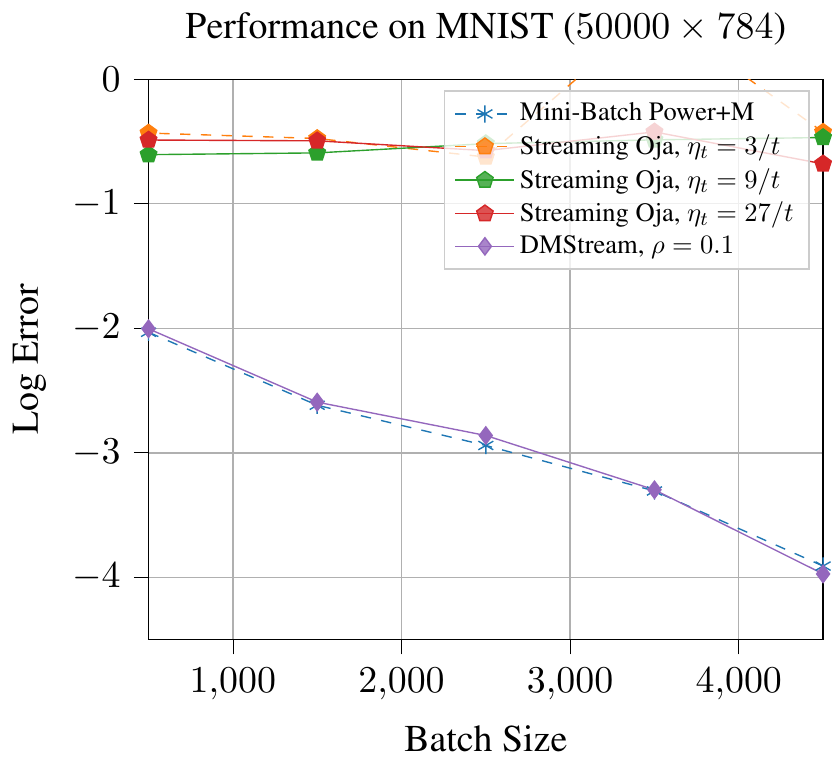}
	}
\end{minipage}\hfill
\begin{minipage}[c]{0.50\textwidth}
\caption{\textbf{Performance by batch size}. In this experiment we measure the performance of DMStream versus Oja's algorithm (with various step sizes) and Mini-Batch Power+M (optimal $\beta=\lambda_2^2/4$). We measure performance by the log error $\log_{10}(1 - \frac{||X^\top q_K||}{||X^\top v_1||})$. DMStream exhibits a consistent increase in accuracy as batch-size is increased and mimics the performance of optimal Mini-Batch Power+M.}  
\label{fig:stream}
\end{minipage}
\end{figure}
\paragraph{Results}
Our algorithm performs significantly better than Oja's algorithm for several step sizes and is as accurate as Mini-Batch Power+M initialized with optimal $\beta=\lambda_2^2/4$. As Table \ref{table:ten} demonstrates, our accuracy eventually ceases to improve for fixed batch size, but Figure \ref{fig:stream} indicates noticeable improvement with increasing batch size, owing to reduced variance of our unbiased estimates $\widehat{A}_t$.

\subsection{Application: Spectral Clustering} 
\paragraph{Overview}
Clustering is the unsupervised learning task of dividing a collection of data points into distinct groups or "clusters." The k-means algorithm \protect{\citep{lloyd1982least}} is a popular clustering method which has found ubiquitous use in machine learning, including social network analysis \protect{\citep{mishra2007clustering}}, image processing \protect{\citep{shi2000normalized}}, and other data mining tasks. 

However, k-means is limited in effectiveness when applied to nonlinear data. Spectral clustering is an extension of k-means used to properly separate nonlinear data. Whereas k-means is applied directly on the data points $\{x_i\}_{i=1}^n$, spectral clustering first begins with a symmetric affinity matrix $A_{ij}=s(x_i,x_j)$ where $s$ is a similarity function which could be Euclidean distance, for example. We define the diagonal matrix $D$ where $D_{ii}=\sum_{j=1}^n A_{ij}$ and then form the normalized affinity matrix $W=D^{-1}A$. Spectral clustering then computes the top $k$ components of $W$. These eigenvectors are supplied to the k-means algorithm to obtain the final separation result. 

The power iteration may be used to find the eigenvectors of $W$, which is referred to as power iteration clustering (PIC). The deflation-based PIC algorithm \protect{\citep{thang2013deflation}} we use for our experiments is provided in Appendix \mbox{\ref{app:spectralcluster}}. In this section, we compare variants of the power iteration in carrying out PIC and demonstrate that DMPower is capable of faster eigenvector computation and more accurate data separation when compared against the vanilla power iteration.

\captionsetup[figure]{labelfont=,}

 \begin{figure}[!htbp]
 \floatconts
 {fig:cluster}
 {\caption{\textbf{Performance of DMPM on Spectral Clustering.} We depict the performance of DMPM ($\rho=\sqrt[3]{\epsilon}$) on dividing a collection of data points into their natural, geometrically-partitioned clusters. The first series (a-e) is the concentric circles dataset while the second series (f-j) is the half-moons dataset. In both series, the first image depicts the unlabeled arrangement, the second depicts a na{\"i}ve application of k-means without spectral clustering, while the last three images depict the performance of k-means assisted by spectral clustering over progressively tighter $\epsilon$ error thresholds (thus, requiring more accurate affinity matrix eigenvector approximations). As $\epsilon$ grows smaller, the data separation improves, eventually achieving perfect classification by $\epsilon=10^{-8}$. }}
 {
 \subfigure[\small{Unlabeled}]{
 \resizebox{0.19\textwidth}{!}{\includegraphics[]{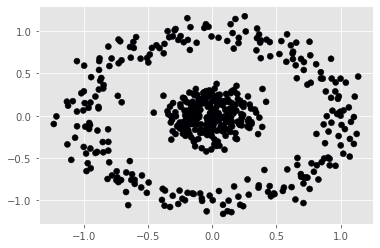}} 
 }
\hspace*{-3mm}
\subfigure[\small{Na{\"i}ve k-means}]{
\resizebox{0.19\textwidth}{!}{\includegraphics[]{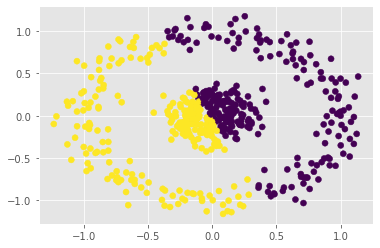}}
}
\hspace*{-3mm}
\subfigure[\small{$\epsilon=10^{-4}$}]{
\resizebox{0.19\textwidth}{!}{\includegraphics[]{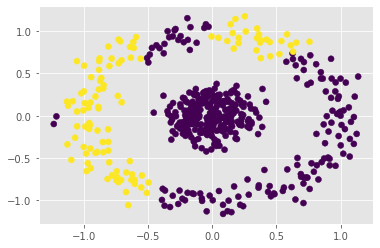}}
}
\hspace*{-3mm}
\subfigure[\small{$\epsilon=10^{-6}$}]{
\resizebox{0.19\textwidth}{!}{\includegraphics[]{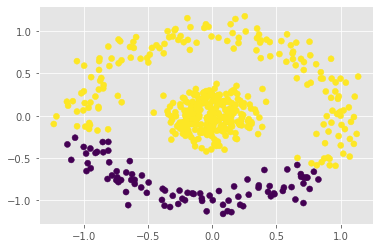}}
}
\hspace*{-3mm}
\subfigure[\small{$\epsilon=10^{-8}$}]{
\resizebox{0.19\textwidth}{!}{\includegraphics[]{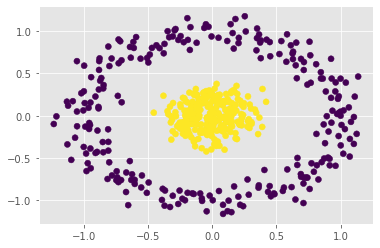}}
}
\subfigure[\small{Unlabeled}]{
 \resizebox{0.19\textwidth}{!}{\includegraphics[]{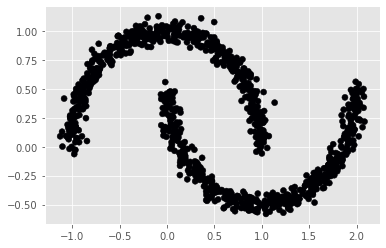}} 
 }
\hspace*{-3mm}
\subfigure[\small{Na{\"i}ve k-means}]{
\resizebox{0.19\textwidth}{!}{\includegraphics[]{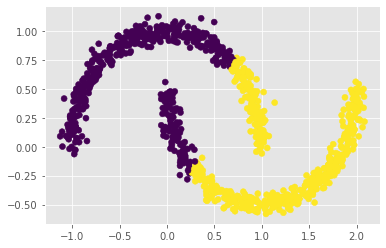}}
}
\hspace*{-3mm}
\subfigure[\small{$\epsilon=10^{-4}$}]{
\resizebox{0.19\textwidth}{!}{\includegraphics[]{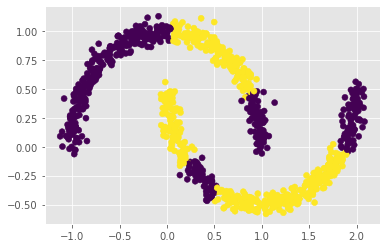}}
}
\hspace*{-3mm}
\subfigure[\small{$\epsilon=10^{-6}$}]{
\resizebox{0.19\textwidth}{!}{\includegraphics[]{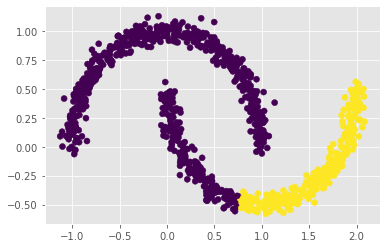}}
}
\hspace*{-3mm}
\subfigure[\small{$\epsilon=10^{-8}$}]{
\resizebox{0.19\textwidth}{!}{\includegraphics[]{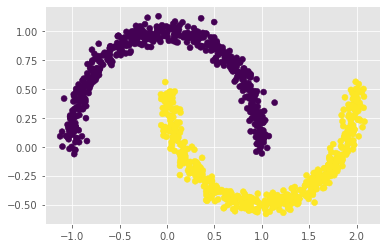}}
}
 }
 \end{figure}



\paragraph{Experimental Setup} 
We used two popular toy datasets for clustering: half-moons and concentric circles. The half-moons dataset was generated with 500 samples while the concentric circles set was generated using 1000 samples. We compared DMPower (with various $\rho$ settings) against the vanilla power method and Power+M with optimal assignment of $\beta$. We follow the deflation-based approach of spectral clustering, which is outlined in Algorithm \mbox{\ref{alg:piccluster}}. Each method was used to recover the top two eigenvectors and eigenvalues of the normalized affinity matrix $W$ of each dataset (our similarity function was pairwise $\ell_2$ distance). 

We use a practical implementation of the DMPower Algorithm \mbox{\ref{alg:dmpm}}: we exit the pre-momentum phase when $||w_{j+1}-w_j||\leq \rho$, where $\rho$ is a function of $\epsilon$, and exit the momentum phase when $||q_{k-1}-q_{j}||\leq \epsilon$. Similarly, for the vanilla power iteration Algorithm \mbox{\ref{vanpm}} and the alternative Power+M update in equation \mbox{\ref{powerm_update}}, we end the procedure once $||q_{k+1}-q_k||\leq \epsilon$. Therefore, termination of these algorithms is governed by the closeness of the approximates.

\paragraph{Results}
DMPower, under all $\rho$ settings, and for most error thresholds $\epsilon$ requires fewer iterations to recover the top two eigenvectors of the affinity matrices than the vanilla power method, see Table \mbox{\ref{table:ten}}. Furthermore, DMPower with $\rho=\sqrt{\epsilon},\sqrt[3]{\epsilon}$ performs similarly in both iteration complexity and accuracy when compared against Power+M with optimal $\beta=\lambda_2^2/4$ as reported in to Tables \mbox{\ref{table:eleven}} and \mbox{\ref{table:thirteen}}. Figure \mbox{\ref{fig:cluster}} depicts a progression of separation over the datasets, when DMPower ($\rho=\sqrt[3]{\epsilon}$) is used.
\section{Conclusion}\label{sec:conc}
In summary, this paper introduces a new scheme for accelerating the vanilla and streaming power methods. The realization of this scheme is the delayed momentum power method (DMPower) and its streaming companion the delayed streaming momentum method (DMStream). DMPower is experimentally shown to outperform the vanilla power iteration and achieves iteration complexity similar to an existing accelerated method Power+M initialized with optimal hyperparameters. Empirically, it also outperforms the state-of-the-art Lanczos algorithm in both wall-time and iteration complexity under several regimes. DMStream is shown to vastly outperform all variations of Oja's algorithm and register close-to-optimal error when compared against the Mini-Batch Power+M initialized with an optimal momentum coefficient. We provide convergence guarantees for DMPower and DMStream using a mixture of perturbation theory and classical power iteration bounds. While other momentum-based methods rely on unrealistic spectral knowledge for acceleration, DMPower and DMStream are both practical and fast. 

\section*{Acknowledgements}
Rabbani thanks John Mattox and Mark Davis for helpful discussions. Huang is supported by a startup fund from the Department of Computer Science of the University of Maryland, National Science Foundation IIS-1850220 CRII Award 030742-00001, DOD-DARPA-Defense Advanced Research Projects Agency Guaranteeing AI Robustness against Deception (GARD), Laboratory for Physical Sciences at University of Maryland, and Adobe, Capital One and JP Morgan faculty fellowships.
\newpage


\bibliography{supp_bib}

\appendix

\newpage
\section{Preliminaries and Facts}
\label{app-prelim}
\subsection{Vanilla Power Method}
Unless noted otherwise, $\norm{\cdot}$ refers to the 2-norm for vectors and the induced 2-norm for matrices. We recall the vanilla power method algorithm. 
\begin{algorithm}
\caption{Vanilla Power Method}
\label{vanpm}
\begin{algorithmic}[1]
\Require $A\in\mathbb{R}^{d\times d}$ diagonalizable, initial vector $q_0\in\mathbb{R}^d$, error threshold $\epsilon$
\Ensure $\epsilon$-accurate approximation of $v_1,\lambda_1$
\For{$k=1,2,\dots$} \Comment{While $||q_k-v_1||>\epsilon$}
\State $q_k \leftarrow \frac{Aq_{k-1}}{\left\lVert Aq_{k-1}\right\rVert}$
\State $\nu_k \leftarrow q_k^{\top}Aq_k$ \Comment{Rayleigh Quotient}
\EndFor
\Return{$q_k,\nu_k$} 
\end{algorithmic}
\end{algorithm}

Although the termination condition relies on an $\ell_2$ distance from $v_1$, this is often replaced with a sine squared error condition, that is, we exit the loop once $\sin^2(\theta_k)\triangleq 1-(q_k^{\top}v_1)^2>\epsilon$. We make use of both variations throughout this paper.
 \subsection{Power Iteration Bounds}
 The setting for the next two lemmas is the following: we let  $A\in\mathbb{R}^{d\times d}$ be symmetric with spectrum $|\lambda_1|>|\lambda_2|\geq |\lambda_3| \geq \cdots \geq |\lambda_n|$ and associated unit eigenvectors $v_1,v_2,\dots,v_n$. Note that we do not need the presence of the second eigengap $|\lambda_2|>|\lambda_3|$ for these classical results. We will now establish several well-known inequalities regarding the accuracy of the vanilla power iteration.
\begin{lemma}[\cite{quarteroni2010numerical}, p. 194]
\label{vanillal2lemma}
Let $q_0\in\mathbb{R}^d$ such that $|{q_0}^{\top}v_1|\neq 0$. We may write $q_0=\sum_{i=1}^{d}c_iv_i$ since $A$ is diagonalizable (being symmetric). Let $C=\bigl(\sum_{i=1}^{d}\bigl(\frac{c_i}{c_1}\bigr)^2\bigr)^{1/2}$. We have that
\begin{equation}
    \norm{\Tilde{q_k}-v_1}\leq C{\biggl\lvert\frac{\lambda_2}{\lambda_1}\biggr\rvert}^k\hspace{1cm} k\geq 1
\end{equation}
where
\begin{equation}
    \Tilde{q_k}=\frac{q_k}{\norm{A^kq_0}}{\alpha_1{\lambda_1}^k}=v_1+\sum_{i=2}^d\frac{c_i}{c_1}\biggl(\frac{\lambda_i}{\lambda_1}\biggr)^k v_i,\hspace{1cm} k=1,2,\dots
\end{equation}
\end{lemma}
\textit{Remark.} Since $\Tilde{q_k}$ is nothing more than a scaled version of $q_k$ convenient for analysis, as an abuse of notation, where 2-norm inequalites are invoked involving $q_k$ and $v_1$, we assume we are working with $\Tilde{q_k}$. 
\begin{lemma}[\cite{golub2012matrix}, p. 451]
\label{vanillarayleighbound}
 Let $q_0\in\mathbb{R}^n$ such that $|{q_0}^\top v_1|\neq 0$. Let $q^k=\frac{A^kq_0}{\norm{A^kq_0}}$ and $\nu_k={q_k}^\top A{q_k},$ i.e., the basic power iteration approximation of $v_1$ and $\lambda_1$ after $k$ steps. Define $\theta_k \in [0,\pi/2]$ by $\cos(\theta_k)=|{q_k}^\top v_1|$. For $k=0,1,2,\dots$, we have
\begin{align}
|\sin(\theta_k)|&\leq \tan(\theta_0)\left|\frac{\lambda_2}{\lambda_1}\right|^k\\
|\lambda_1-\nu_k|&\leq \underset{2\leq i \leq d}{\max} |\lambda_1-\lambda_i|\tan(\theta_0)^2 \left|\frac{\lambda_2}{\lambda_1}\right|^{2k}
\end{align}
\end{lemma}
We now prove a more general, well-known fact: the Rayleigh quotient is a quadratically-accurate estimate when compared to the sine squared error. 
\begin{lemma}
\label{general_rayleigh_bound}
Let $w_k\rightarrow v_j$ for $\{w_k\}_{i=1}^\infty$ a sequence of unit vectors and $v_j$ a unit eigenvector of $A$ associated to eigenvalue $\lambda_j$. Let $\eta_k$ be the Rayleigh quotient of $w_k$ and $\sin^2(\gamma_k)=1-(w_k^\top v_j)^2$. Then
\begin{equation}
|\eta_k-\lambda_i|=\mathcal{O}\bigl(\sin^2(\gamma_k) \bigr)
\end{equation}
\end{lemma}

\begin{proof}
Express $w_k=\sum_{i=1}^d c_iv_i$, i.e., as a linear combination in the eigenbasis of $A$. 
We have then that 
\begin{equation}
    \eta_k-\lambda_j=w_k^\top Aw_k-\lambda_j=\frac{\sum_{i=1}^d c_i^2\lambda_i}{\sum_i^d c_i^2}-\lambda_j=\frac{\sum_{i=1}^d(\lambda_i-\lambda_j)c_i^2}{\sum_{i=1}^d c_i^2}
    \leq\underset{i\neq j}{\max}|\lambda_j-\lambda_i|\sum_{i\neq j}^d c_i^2
\end{equation}
Since $\sum_{i=1}^d c_i^2=1$ ($w_k$ is unit), we have that $\sum\limits_{\substack{i=1 \\ i\neq j}}^d c_i^2=1-c_j^2=1-(w_k^{\top}v_j)^2=\sin^2(\gamma_j)$. Therefore,
\begin{equation}
    |\eta_k-\lambda_j|\leq \underset{i\neq j}{\max}|\lambda_j-\lambda_i|\sin^2(\gamma_k)
\end{equation}
proving our claim.
\end{proof}

\section{Noisy Deflation}\label{app:noise-bounds}
The following results will help us determine how many steps of inexact deflation we need to complete before we achieve small enough noise at each iteration to permit convergence. As usual, we let $\nu_k$ and $q_k$ reflect our eigenvalue and eigenvector approximation at step $k$ of a vanilla power iteration with matrix $A$. We first recall the formulation of a \textit{noisy power method} (NPM). We consider an alternative update step of Algorithm \ref{vanpm}:
\begin{equation}
    q_{k}=\frac{Aq_{k-1}+G_{k-1}}{\norm{Aq_{k-1}+G_{k-1}}}
\end{equation}
Here, $G_{k-1}$ is noise or a perturbation added at each round. The following is a powerful result on the conditions under which NPM can successfully converge:

\begin{corollary}[Noisy Power Method Convergence \cite{hardt2014noisy}]
Let $k\leq p$. Let $U\in\mathbb{R}^{d\times k}$ represent the top $k$ singular vectors of $B\in\mathbb{R}^{d\times d}$ and let $\sigma_1\geq \sigma_2\geq \cdots \geq \sigma_d$ denote its singular values. Suppose $X_0$ is an orthonormal basis of a random $p$-dimensional subspace. Further, suppose at every step of NPM we have
\begin{equation}
\label{hardt-price-bounds}
5||G_{\ell}||\leq \epsilon(\lambda_k-\lambda_{k-1})\textrm{\hspace{0.25cm}and\hspace{0.25cm}} 5||U^{\top}G_{\ell}||\leq (\lambda_k-\lambda_{k-1})\frac{\sqrt{p}-\sqrt{k-1}}{\tau\sqrt{d}}
\end{equation}
for some fixed parameter $\tau$ and $\epsilon<1/2$. Then with all but $\tau^{\Omega(p+1-k)}+e^{-\Omega(d)}$ probability, there exists an $L=\mathcal{O}\bigl(\frac{\lambda_k}{\lambda_k-\lambda_{k-1}}\log\frac{d\tau}{\epsilon}\bigr)$ so that after $L$ steps we have $\norm{(I-X_LX_L^{\top})U}\leq \epsilon$.
\end{corollary}
In the context of Theorem \ref{thm:main-DMPM}, we have that $B=A-\lambda_1v_1v_1^{\top}$, $X_0=w_0$, $U=v_2$, and $G_{\ell}$ is the error $||\lambda_1v_1v_1^{\top}-\nu_{\ell}q_{\ell}q_{\ell}^{\top}||$, that is, the "inexactness" of our deflation, and $d$ is the dimension. We also note that in relation to $A$, the deflated matrix $B$ has spectrum $\lambda_2 > \lambda_3 \geq \cdots \lambda_{n-1} \geq \lambda_n=0$. In this setting, we have that $\norm{(I-X_LX_L^{\top})U}=\norm{(I-w_Lw_L^{\top})v_2)}=\sqrt{1-(w_Lv_2^{\top})^2}$, i.e., the sine error between $v_2$ and $w_L$. We will now provide an upper bound on $||G_{\ell}||$ and show it decays with every round. For notational convenience, we denote $\theta_0:= \arccos{|q_0^\top v_1|}$. 

\begin{lemma}
\label{power-perturbation}
 If we express $A-{\nu_k}q_kq_k^{\top}$ as $A-\lambda_1v_1v_1^{\top}+G_k$, where $G_k=\lambda_1v_1v_1^{\top}-\nu_{\ell}q_{\ell}q_{\ell}^{\top}$ is a perturbation reflecting the error of our eigenpair approximation, then $\norm{G_k}=\mathcal{O}(\tan^2 \theta_0{\lvert\frac{\lambda_2}{\lambda_1}\rvert}^k)$.
\end{lemma}
\begin{proof}
Let $q_k=v_1+\xi_k$ and $\nu_k=\lambda_1+\gamma_k$, that is, $\xi_k$ and $\gamma_k$ reflect the perturbations (error) associated with our power iterates. We may express our inexact deflation matrix as follows: 
\begin{equation}
    A-\nu_k{q_k}{q_k}^\top=A-(\lambda_1+\gamma_k)(v_1+\xi_k)(v_1+\xi_k)^\top.
\end{equation}
Expanding, we have that our deviation from $A-\lambda_1v_1v_1^{\top}$ is expressible as
\begin{equation}
    G_k = \lambda_1(v_1\xi_k^\top+\xi_kv_1^{\top}+\xi_k\xi_k^{\top})+\gamma_k(v_1v_1^{\top}+v_1\xi_k^\top+\xi_kv_1^{\top}+\xi_k\xi_k^{\top}).
\end{equation}
\begin{align}
    ||G_k||\leq |\lambda_1|(||v_1\xi_k^\top||+||\xi_kv_1^{\top}||+||\xi_k\xi_k^{\top}||)+|\nu_k|(||v_1v_1^{\top}||+||v_1\xi_k^\top||+||\xi_kv_1^{\top}||+||\xi_k\xi_k^{\top}||)
\end{align}
Express the initial vector $q_0$ of the power iteration in the eigenbasis of $A$: $v_0=\sum_{i=1}^n c_i v_i$. For the next step, we recall by the previous two lemmas that
\begin{equation*}
|\lambda_1-\nu_k|\leq \underset{2\leq i \leq d}{\max}|\lambda_1-\lambda_i|\tan^2(\theta_0)\left|\frac{\lambda_2}{\lambda_1}\right|^{2k}
\end{equation*}
and 
\begin{equation*}
||q_k-v_1||\leq \left|\frac{\lambda_2}{\lambda_1}\right|^k \bigl(\sum_{i=1}^{d}(\frac{c_i}{c_1})^2\bigr)^{1/2}.
\end{equation*}
For brevity, we denote $C=\bigl(\sum_{i=1}^{d}[\frac{c_i}{c_1}]^2\bigr)^{1/2}$ and $D=\underset{2\leq i \leq d}{\max}|\lambda_1-\lambda_i|\tan^2(\theta_0)$. Now, it also well known that for $u, v \in \mathbb{R}^l$ that $||uv^t||=|u^{\top}v|$. Therefore, in conjunction with application of the Cauchy-Schwarz inequality, we have that
\begin{align}
||v_1v_1^{\top}||&=1\\
||v_1\xi_k^{\top}||=||\xi_kv_1^{\top}||=|v_1^\top \xi_k|\leq |\xi_k|&\leq C \left|\frac{\lambda_2}{\lambda_1}\right|^k\\
||\xi_k\xi_k^{\top}||=|\xi_k^{\top}\xi_k|=|\xi_k|^2 &\leq C^2\left|\frac{\lambda_2}{\lambda_1}\right|^{2k}
\end{align}
Noting that $\lambda_1\leq 1$, we have then that
\begin{align}
    ||G_k||\leq 2C\left|\frac{\lambda_2}{\lambda_1}\right|^k+C^2\left|\frac{\lambda_2}{\lambda_1}\right|^{2k}+D\left|\frac{\lambda_2}{\lambda_1}\right|^{2k}(1+2C\left|\frac{\lambda_2}{\lambda_1}\right|^k+C^2\left|\frac{\lambda_2}{\lambda_1}\right|^{2k}).
\end{align}
Thus, $||G_k||=\mathcal{O}(\tan^2(\theta_0){\lvert\frac{\lambda_2}{\lambda_1}\rvert}^k)$.
\end{proof}

\begin{lemma}
\label{rayleighdeflate}
Let $v,w\in\mathbb{R}^d$ be of unit length, $A\in\mathbb{R}^{d\times d}$. We have then that
\begin{equation}
    |v^{\top}Aw|\leq ||A||_F
\end{equation}
\end{lemma}
\begin{proof}
Applying the Cauchy-Schwarz inequality and sub-multiplicativity of the induced 2-norm for matrices,
\begin{equation*}
    |v^{\top}Aw|\leq ||v||\cdot||A||\cdot||w||_2=||A||_2\leq ||A||_F.
\end{equation*}
\end{proof}
We may establish how many initial steps of inexact deflation we need to run before it is set on a path towards convergence. Our setting is the same as in the previous lemmas. 
\begin{theorem}
\label{hardt-price-boundsthm}
Let $A\in\mathbb{R}^{d\times d}$ be symmetric PSD and $w_0$ be the initial unit vector for inexact deflation such that ${w_0}^{\top}x_2\neq 0$. Fix $\tau>1$, $\rho<1/2$ and $\delta=\min\{\rho,\frac{1}{\tau\sqrt{d}}\}$. Then after $T=\mathcal{O}\bigl(\frac{1}{\lambda_1-\lambda_2}\log(\frac{\tan^2 \theta_0}{\delta(\lambda_2-\lambda_3)})\bigr)$ steps, the perturbations $G_k$ achieve the Hardt-Price bounds in Equation \ref{hardt-price-bounds}. 
\end{theorem}
\begin{proof}
Let $B=A-\lambda_1v_1{v_1}^{\top}$. For $k=1,2,\dots$ we have that step $k$ in the pre-momentum phase involves computing $q_k=\frac{A^kq_{k-1}}{\norm{A^kq_{k-1}}},\nu_k=q_k^{\top}Aq_k$ and then an inexact deflation $(A-\nu_kq_k^{\top}q_k)w_{k-1}$. We have then that inexact deflation at step $k$ is representable as
\begin{equation}
(B+G_k)w_{k-1}=Bw_{k-1}+G_kw_{k-1}
\end{equation}
where $G_k$ is associated with the error $\norm{\lambda_1v_1{v_1}^{\top}-\nu_kq_k{q_k}^{\top}}$. By Lemma \ref{power-perturbation}, the fact the $w_k$ are unit, and sub-multiplicativity, we have that
\begin{equation}
||G_kw_{k-1}||\leq ||G_k||\cdot||w_{k-1}||=||G_k||= \mathcal{O}\Bigl(\tan^2\theta_0\left|\frac{\lambda_2}{\lambda_1}\right|^k\Bigr).
\end{equation}
Reminding ourselves of the Hardt-Price bounds, we need any given perturbation to satisfy
\begin{equation}
    5||G_k||\leq \rho(\lambda_2-\lambda_3) \hspace{0.5cm} and \hspace{0.5cm} 5||G_k||\leq (\lambda_2-\lambda_3)\frac{1}{\tau\sqrt{d}}
\end{equation}
Solving for $k$ such that $\tan^2\theta_0\left|\frac{\lambda_2}{\lambda_1}\right|^k$ satisfies these bounds, we arrive at our claim.

\end{proof}
\section{Proof of Theorem~\ref{thm:main-DMPM} }\label{app:proof-dmpm}
We first outline the setting before proving our main result. We let $A\in\mathbb{R}^{d\times d}$ be symmetric PSD with spectrum
\begin{equation}
1\geq\lambda_1>\lambda_2 >\lambda_3 \geq \cdots \geq \lambda_d \geq 0. 
\end{equation}
  Distinct to our setting is the presence of a positive eigengap between $\lambda_2$ and $\lambda_3$. We select $q_0=\sum_{i=1}^n c_iv_i$ such that $|v_1^{\top}q_0|\neq 0$, which will be used for the vanilla and momentum power iterations, and $w_0=\sum_{i=1}^{d}b_iv_i$ such that $|{v_2}^{\top}w_0|\neq 0$, which will be used for inexact deflation. Let $\theta_0=\arccos{|q_0^\top v_1|}$. Lastly, in the first for-loop we add the theoretical termination condition $\texttt{while } |\lambda_2-\mu_k|>\rho$.
  
\begin{theorem}[Restating of Theorem~\ref{thm:main-DMPM}]\label{thm:main-DMPM-again}
Let $J$ represent the number of steps in the pre-momentum phase and $K$ the number of steps in the momentum phase as in Algorithm \ref{alg:dmpm}. Let $\epsilon<1$ represent the desired error threshold of our $v_1$ estimates, i.e., $\sin^2\theta(q_t,v_1)<\epsilon$ and $\rho<\min\{1/2,\sqrt{\frac{\lambda_1-\lambda_2}{\lambda_2-\lambda_d}}\}$ represent the desired error threshold of our $\lambda_2$ estimates, i.e., $|\mu_k-\lambda_2|<\rho$. Select unit $q_0\in\mathbb{R}^d$ and $w_0\in\mathbb{R}^d$. Further fix $\tau>1$ and $\delta=\min\{\rho,\frac{1}{\tau \sqrt{d}}\}$. Then after
\begin{align}
J&=\mathcal{O}\bigl(\frac{1}{\lambda_1-\lambda_2}\log\frac{\tan^2\theta_0}{\delta(\lambda_2-\lambda_3)}+\frac{\lambda_2}{\lambda_2-\lambda_3}\log\frac{d\tau}{\rho}\bigr),\\ K&=\mathcal{O}\bigl(\frac{\beta}{\sqrt{\lambda_1^2-4\beta^2}}\log\frac{1}{\epsilon}\bigr) 
\end{align}
pre-momentum and momentum steps, respectively, where $\beta = {\widehat{\lambda}_2}^2/4 = \mu_J^2/4$, with all but $\tau^{-\Omega(1)}+e^{-\Omega(d)}$ probability, \emph{DMPower} outputs a vector $q_{K}$ with 
\begin{equation}
 \sin^2 \theta(q_{K},v_1) <\epsilon. 
\end{equation}
\end{theorem}
\begin{proof}
We divide our proof into an analysis of the pre-momentum stage and then the momentum stage. Specifically, we will first establish how many iterations are needed to acquire a $\beta\in[\lambda_2^2/4,\lambda_1^2/4)$. Then, we will use this $\beta$ to conduct Power+M updates and derive how many additional steps we will need to obtain a $q_k$ with $1-(q_k^{\top}v_1)^2<\epsilon$. 

\emph{Pre-momentum phase:} By Theorem \ref{hardt-price-boundsthm}, after $J_1=\mathcal{O}\bigl(\frac{1}{\lambda_1-\lambda_2}\log\frac{\tan^2 \theta_0}{\delta(\lambda_2-\lambda_3)}\bigr)$ steps, we achieve the Hardt-Price bounds. Therefore, by Corollary 1.1 in \citep{hardt2014noisy}, after a further $J_2=\mathcal{O}\bigl(\frac{\lambda_2}{\lambda_2-\lambda_3}\log\frac{d\tau}{\rho} \bigr)$ iterations the $w_t$ will converge to $\rho$-accuracy. That is, if we let $J=J_1+J_2$ we have
\begin{equation}
\norm{(I-w_Jw_J^{\top})v_2}=1-(w_J^{\top}v_2)=\sin(\gamma_J)<\rho   
\end{equation}
 with all but $\tau^{-\Omega(1)}+e^{-\Omega(n)}$ probability. Therefore in conjunction with Lemma \ref{general_rayleigh_bound} we have that
 \begin{equation}
    |\lambda_2-\mu_J|\leq (\lambda_2-\lambda_d)\sin^2(\gamma_J)<(\lambda_2-\lambda_d)\rho^2<\lambda_1-\lambda_2
 \end{equation}
so by Proposition \ref{prop-rhoprecision} we obtain
\begin{equation}
\frac{1}{4}|\lambda_2^2-\mu_J^2|<\frac{1}{4}|\lambda_1^2-\lambda_2^2|.    
\end{equation} We set $\beta=\frac{\mu_J^2}{4}$ as our momentum coefficient and proceed to the momentum phase.

\emph{Momentum phase:} Our momentum coefficient is now within the interval of acceleration$^\dagger$, i.e., $\beta\in[\lambda_2^2/4,\lambda_1^2/4)$, so we may now invoke Theorem \ref{PowerMThm}, which tell us that after
\begin{equation}
   K=\mathcal{O}\biggl( \frac{\beta}{\sqrt{\lambda_1^2-4\beta^2}}\log\frac{1}{\epsilon}\biggr) 
\end{equation}
steps of Power+M iteration on $q_J$ (which we now take to be our initial vector for Power+M), we will have that $1-(q_{J+K}^{\top}v_1)^2<\epsilon$, completing our proof.

$\dagger$ Subtly, we assume that $\mu_J\geq\lambda_2$. We discuss the case $\mu_J < \lambda_2$ in Theorem \ref{PowerMThm-Restated}. 
\end{proof}

\section{Proof of Theorem \ref{thm:main-DMSPM}}
\label{app:proof-dm-spm}

\begin{algorithm}[!ht]
   \caption{Stochastic Power Method/Streaming PCA
   }
\label{alg:streamPCA}
\begin{algorithmic}[1]
\Require Streaming inputs $x_1, x_2, \dots \in\mathbb{R}^{d}$, batch size $n$, unit $q_0\in\mathbb{R}^d$, iterations $T$, unit $q_0\in\mathbb{R}^d$
\For{$t=1,2,\dots, T$}
\State Generate unbiased estimate $\widehat{A}_t=\frac{1}{n}\sum_{i=(t-1)n+1}^{tn} x_{i}x_{i}^\top $
\State $q_t \leftarrow \widehat{A}_tq_{t-1}$
\State $q_t \leftarrow q_t/\norm{q_t}$
\State $\nu_t \leftarrow q_t^{\top}\widehat{A}_tq_t$ \Comment{Inexact Rayleigh quotient estimate of $\lambda_1$}
\EndFor
\Return{$q_T,\nu_T$} 
\end{algorithmic}
\end{algorithm}

We re-outline the typical streaming setting: we have $d$-dimensional data points $x_1, x_2, \dots \sim \mathcal{D}$, with underlying covariance matrix $A\in\mathbb{R}^{d \times d}$ with eigenvalues $1\geq \lambda_1 > \lambda_2 > \lambda_3 \geq \lambda_4 \geq \dots \geq \lambda_n \geq 0$.   
Presumably, it is too costly to access and/or store $A$, but we have access to a stream of inputs $x_1, x_2, \dots$. Algorithm \ref{alg:streamPCA} is a conventional streaming PCA method designed to recover the principal components of $A$ using only a sample of inputs. As an abuse of notation, they are not indexed in any particular order -- we receive the data points in a uniformly random manner. At each round, we form an unbiased estimate $\widehat{A}=\frac{1}{n}\sum_{i=1}^n x_ix_i^\top$, where $n$ is a fixed batch size. We are interested in determining how many total samples are needed to output a vector $q_t$ such that $\sin^2(q_t,v_1)<\epsilon$ for a fixed precision $\epsilon<1$, which is also referred to as the sample complexity. 

Similar to the vanilla power method, one may accelerate a conventional streaming PCA by attaching a momentum term. That is, our update step in Algorithm $\ref{alg:streamPCA}$ would instead be 
\begin{equation}
\label{momentum-stream}
    q_{j+1} \leftarrow \widehat{A}_jq_{j} - \beta q_{j-1}
\end{equation}
where $\beta$ is a momentum coefficient. De Sa et. al provide a guarantee for updates of this kind:
\begin{theorem}[\cite{de2018accelerated}, Theorem 3]
\label{momentum-stream-desa}
Suppose we run Algorithm \ref{alg:streamPCA} with momentum updates as in equation \ref{momentum-stream}. Let $\Sigma=\mathbb{E}[(\widehat{A}_j-A)\otimes(\widehat{A}_j-A)]$. Assume we initialize with unit $q_0\in\mathbb{R}^d$ and with $|v_1^\top q_0| \geq 1/2$. For any $\delta <1$ and $\epsilon<1$, if $2\sqrt{\beta}\in [\lambda_2, \lambda_1)$ and
\begin{equation}
    J=\frac{\sqrt{\beta}}{\sqrt{\lambda_1-4\beta}}\log\Bigl(\frac{32}{\delta\epsilon}\Bigr) \hspace{0.2cm} and \hspace{0.2cm} ||\Sigma||\leq \frac{(\lambda_1^2-4\beta)\delta\epsilon}{256\sqrt{d}J} =
    \frac{(\lambda_1^2-4\beta)^{3/2}\delta\epsilon}{256\sqrt{d}\sqrt{\beta}}\log^{-1}\Bigl( \frac{32}{\delta\epsilon} \Bigr),
\end{equation}
then after $J$ updates with probability at least $1-2\delta$, we have that $\sin^2(q_J,v_1)\leq \epsilon$.
\end{theorem}
Similar to Power+M, streaming PCA with momentum updates experiences noticeable acceleration. However, of particular note is the bounded variance condition and the initialization of $\beta$. The former has been addressed through further variance reduction techniques outlined in \citep{de2018accelerated}, but current literature does not suggest how to intelligently set $\beta$ so that it lies in the convergence interval $[\lambda_2^2/4, \lambda_1^2/4)$. Similar to our design of DMPower, our Algorithm \ref{alg:dmspm} successively approximates a convergent $\beta$ in a pre-momentum phase before dropping into a momentum phase. The convergence analysis of DMStream is more challenging, however, since we have two sources of noise: the estimation error $||A-\widehat{A}||$ and the estimation error of $||v_1-q_t||$. The remainder of this section is devoted to providing upper bounds on both sources of noise. 

We have the following noisy representation of $\widehat{A}x$ for $x\in\mathbb{R}^d$:
\begin{equation}
\label{noisy-streaming-rep}
    \widehat{A}x=Ax + H
\end{equation}
where $H=(A-\widehat{A})x$. In light of Theorem \ref{hardt-price-boundsthm}, if we can control $||H||$, then we can produce a convergent noisy power method. To this end, we provide several results to aid our analysis, some of which are not proven here. The first result bounds the error in a single step of a streaming PCA:
\begin{lemma}[\cite{hardt2014noisy}, Lemma 3.5]
\label{lemma-sizeG}
Let $A$ be a covariance matrix as in the setting described above and $\widehat{A}=\frac{1}{n}\sum_{i=1}^{n}x_ix_i^\top$ an empirical estimate based off a streaming batch $x_1, x_2, \dots, x_n$. Consider the noisy representation as in equation \ref{noisy-streaming-rep}. Then with all but $\mathcal{O}(1/n^2)$ probability
\begin{equation}
    ||H||\leq\sqrt{\frac{\log^4 n\log d}{n}}+\frac{1}{n^2}\hspace{0.3cm} and \hspace{0.3cm} ||v_1^\top H||\leq \sqrt{\frac{\log^4n\log d}{n}+\frac{1}{n^2}}
\end{equation}
\end{lemma}
We now consider $H_t=(A-\widehat{A}_t)q_{t-1}$, where $\widehat{A}_t$ and $q_t$ are the unbiased estimate of $A$ and $v_1$ respectively in round $j$ as in Algorithm \ref{alg:streamPCA}.

\begin{proposition}[\cite{hardt2014noisy}, Theorem 3.2]
\label{dmspm-bounds}
Choose batch size $n$ such that
\begin{equation}
\frac{n}{\log^4 n}=\mathcal{O}\Bigl(\frac{1/\epsilon^2\log d}{(\lambda_1-\lambda_2)^2 d} \Bigr)
\end{equation}
for $\epsilon<1/2$ and we will have by Lemma \ref{lemma-sizeG} that
\begin{equation}
    ||H_t||\leq \frac{\epsilon(\lambda_1-\lambda_2)}{5}\hspace{.3cm} and \hspace{.3cm} ||v_1^\top H_t||\leq \frac{\lambda_1-\lambda_2}{5\sqrt{d}},
\end{equation} thereby satisfying the Hardt-Price bounds as in equation \ref{hardt-price-bounds}. Thus, by Theorem $\ref{hardt-price-boundsthm}$, after $T=\mathcal{O}(\log(d/\epsilon)/(1-\lambda_2/\lambda_1))$ iterations, we have with all but $1-\max\{1,T/n^2\}$ probability that Algorithm \ref{alg:dmspm} outputs $q_T$ such that $1-(v_1^\top q_T)^2<\epsilon$. 
\end{proposition}

Now, consider the inexact update in Algorithm \ref{alg:dmspm}, $(\widehat{A}_t-\nu_tq_tq_t^\top)w_t$. For convenience, we let $B:=A-\lambda_1v_1v_1^{\top}$, the exact deflation matrix. We may express this as 
\begin{equation}
\label{fullnoise-streaming}
    (\widehat{A}_t-\nu_tq_tq_t^\top)w_{t-1}=Bw_{t-1}+H_t+G_t
\end{equation}
where $H_t=(A-\widehat{A}_t)w_{t-1}$ and $G_t=(\lambda_1v_1v_1^{\top}-\nu_t q_t q_t^{\top})w_{t-1}$. By Proposition \ref{dmspm-bounds} we know how to control $||H_t||$, so we will now focus our attention on $||G_t||$. We will conduct analysis in the same spirit as Lemma $\ref{power-perturbation}$, but first we will require a another pair of results. The first lemma demonstrates that if we satisfy the Hardt-Price bounds of equation \ref{hardt-price-bounds}, then $\tan \theta(v_1,q_t)$ decreases multiplicatively with each step of a noisy power method.
\begin{lemma}[\cite{hardt2014noisy}, Lemma 2.3 (modified)]
\label{noisy-rounderror}
Let $v_1 \in \mathbb{R}^d$ be the dominant eigenvector with eigenvalue $\lambda_1$ of $A$, $\lambda_2$ the second dominant eigenvalue, $x\in \mathbb{R}^d$ a unit vector. Let $G \in \mathbb{R}^d$ and $\theta_0 = \arccos{|v_1^\top x|}$ satisfy
\begin{align}
    4||v_1^\top G|| &\leq (\lambda_1-\lambda_2)\cos(\theta_0)\\
    4||G|| &\leq (\lambda_1-\lambda_2)\epsilon
\end{align}
for some $\epsilon<1$. Then
\begin{equation}
    \tan \theta(v_1,Ax+G) \leq \max\{\epsilon, \max\{\epsilon,\Bigl(\frac{\lambda_2}{\lambda_1} \Bigr)^{1/4}\}\tan\theta_0\}.
\end{equation}
\end{lemma}
We now provide a bound on the $\sin \theta(v_1,q_t)$ using Lemma \ref{noisy-rounderror}. 
\begin{proposition}[\cite{hardt2014noisy}, Theorem 2.4 (rephrased)]
\label{noisy-sinerror}
Suppose $q_0$ is the initial vector supplied to a noisy power method, $\theta_0=|v_1^\top q_0|$, and $G_t$ is the noise experienced at round $t$. Further suppose that
\begin{align}
    5||v_1^\top G_t|| &\leq (\lambda_1-\lambda_2)\cos \theta_0\\
    5||G_t|| &\leq \epsilon(\lambda_1-\lambda_2)
\end{align}
holds at every stage including and after round $t$ for some $\epsilon< 1/2$. Then
\begin{equation}
  \sin \theta(q_{t+k},v_1)\leq \max\{\epsilon,(\lambda_2/\lambda_1)^{k/4}\tan\theta_0 \}  
\end{equation}
\end{proposition}
\begin{proof}
Let $l \geq t$. Since $|G_l||$ satisfies the modified Hardt-Price bounds of Lemma \ref{noisy-rounderror}, we have that
\begin{equation}
    \tan\theta(v_1,q_l) \leq \max\{\epsilon, \max\{\epsilon,\tan\theta_0\}\}.
\end{equation}
For $\epsilon <1/2$ we have that,
\begin{equation}
\cos \theta(v_1,q_l)\geq \min\{1-\epsilon^2/2,\cos\theta_0 \}\geq \frac{7}{8}\cos\theta_0,    
\end{equation}
which means that we may invoke Lemma $\ref{noisy-rounderror}$ at every step $l\geq t$. This gives us
\begin{equation}
    \tan(v_1, q_{l+1}) = \tan \theta (v_1,Aq_l+G) \leq \max \{\epsilon, \max\{\epsilon,\Bigl(\frac{\lambda_2}{\lambda_1} \Bigr)^{1/4}\}\tan\theta(v_1,q_l)\}.
\end{equation}
Extending this inequality recursively for $l+k$ and noting that $\sin \theta(q_l,v_1) \leq \tan \theta(q_l, v_1)$ gives us our claim. 
\end{proof}
We are now prepared to conduct analysis on $||G_t||$ in equation \ref{fullnoise-streaming}. 
\begin{proposition}
\label{gt-bound}
Choose $\epsilon<1/2$ and $n$ such that
\begin{equation}
\frac{n}{\log^4 n}=\mathcal{O}\Bigl(\frac{1/\epsilon^2\log d}{(\lambda_1-\lambda_2)^2 d} \Bigr).    
\end{equation}.
Let 
\begin{align}
    \phi_t &= |\lambda_1-\lambda_d|\tan^2 \theta(q_0,v_1)\left|\frac{\lambda_2}{\lambda_1}\right|^{2t} + \min\{\frac{\epsilon(\lambda_1-\lambda_2)}{5},\frac{\lambda_1-\lambda_2}{5\sqrt{d}}\}\\ 
    \psi_t &=\sqrt{2-2\sqrt{(1-\min\{1,\max\{\epsilon^2,(\lambda_2/\lambda_1)^{t/2}\tan^2\theta_0 \})\}}}.
\end{align}
Then with all but $\mathcal{O}({1/n^2})$ probability, we have that $\norm{G_t}=\mathcal{O}(\max\{\phi_t, \psi_t \})$.
\end{proposition}
\begin{proof}
We first consider $|\lambda_1-\nu_t|=|\lambda_1-q_{t}^\top\widehat{A}_tq_{t}|$. Let $\theta_0=q_0^\top v_1$. First observe that since $q_t \rightarrow v_1$, we have by Lemma \ref{general_rayleigh_bound} that $|\lambda_1-q_t^\top A q_t|=\mathcal{O}(\sin^2(q_t,v_1))$. We have then that
\begin{align}
    |\lambda_1-\nu_t|=|\lambda_1-q_{t}^\top\widehat{A}_tq_{1}|&=|\lambda_1-q_{t}^\top Aq_{t} + q_{t}^\top Aq_{t} - q_{t}^\top\widehat{A}_tq_{t}|\\
    &\leq |\lambda_1-q_{t}^\top Aq_{t}| + |q_{t}^\top(A -\widehat{A}_t)q_{t}|\\
    &\leq |\lambda_1-\lambda_d|\tan^2 \theta(q_0,v_1)\left|\frac{\lambda_2}{\lambda_1}\right|^{2t} + \min\{\frac{\epsilon(\lambda_1-\lambda_2)}{10},\frac{\lambda_1-\lambda_2}{10\sqrt{d}}\}
\end{align}
where the last inequality follows from applying the result on Rayleigh quotient approximation in Lemma \ref{general_rayleigh_bound} to the left term and Proposition \ref{dmspm-bounds} along with the Cauchy-Schwarz inequality to the right term. 

We now examine $||v_1-q_t||$. We have chosen $n$ in such a way that we satisfy the Hardt-Price bounds of equation \ref{hardt-price-bounds}, therefore, we may invoke Proposition \ref{noisy-sinerror} and conclude that
\begin{equation}
||v_1 - q_t || \leq \sqrt{2-2\sqrt{(1-\min\{1,\max\{\epsilon^2,(\lambda_2/\lambda_1)^{t/2}\tan^2\theta_0 \})\}}},
\end{equation}
where we have used the identity $||v_1-q_t||^2 = 2 - 2\cos\theta(v_1,q_t)$ since unit $v_1$ and $q_t$ are unit vectors. Following the exact same analysis as in the proof of Lemma $\ref{power-perturbation}$, we will arrive at
\begin{equation}
    ||G_t||=\mathcal{O}(\max\{|\lambda_1-\nu_t|,||v_1-q_t||\}).
\end{equation}
Taking the upper bounds on $||\lambda_1-\nu_t||$ and $||v_1-q_t||$ which we derived above gives us our result. 
\end{proof}
We are now prepared to prove our main convergence theorem for DMStream.

\begin{theorem}[Restating of Theorem~\ref{thm:main-DMSPM}]\label{thm:main-DMSPM-again}
Let $\Sigma=\mathbb{E}[(\widehat{A}_t-A)\otimes(\widehat{A}_t-A)]$, where $\widehat{A}_j=\frac{1}{n}\sum_{i=1}^n x_ix_i^\top$ represents any unbiased estimate of $A$ in DMStream with fixed batch size $n$. Assume we initialize with a unit $q_0\in\mathbb{R}^d$ where $d \gg 0$ and $|v_1^\top q_0| \geq 1/2$. Let $\theta_0=\arccos{|q_0^\top v_1|}$. For any $\delta <1$, $\epsilon<1$, suppose
\begin{equation}
     ||\Sigma||\leq \frac{(\lambda_1^2-4\beta)\delta\epsilon}{256\sqrt{d}J} =
    \frac{(\lambda_1^2-4\beta)^{3/2}\delta\epsilon}{256\sqrt{d}\sqrt{\beta}}\log^{-1}\Bigl( \frac{32}{\delta\epsilon} \Bigr),
\end{equation}
where $J$ is the total number of pre-momentum steps we have fixed at runtime. Furthermore, we let $\rho<\min\{1/2,\sqrt{\frac{\lambda_1-\lambda_2}{\lambda_2-\lambda_d}}\}$ represent the error threshold of our $\lambda_2$ estimates, i.e., $|\mu_k-\lambda_2|<\rho$. Lastly, fix $\tau>1$ and $\delta=\min\{\rho,\frac{1}{\tau \sqrt{d}}\}$.
If the batch size $n$ is chosen such that 
\begin{equation}
\frac{n}{\log^4 n}=\mathcal{O}\Bigl(\frac{1/\gamma^2\log d}{(\lambda_2-\lambda_3)^2 d} \Bigr).    
\end{equation}
where $\gamma = \frac{\rho(\lambda_2-\lambda_3)}{10\tau\sqrt{d}}$, then after 

\begin{align}
 J &=\mathcal{O}\Biggl(\frac{1}{\lambda_1-\lambda_2}\log \biggl(\frac{\tan^2 \theta_0\tau\sqrt{d}}{\rho(\lambda_2-\lambda_3)}\biggr) + \frac{\lambda_2}{\lambda_2-\lambda_3}\log\frac{d\tau}{\rho}\Biggr),\\
 K &=\frac{\sqrt{\beta}}{\sqrt{\lambda_1-4\beta}}\log\Bigl(\frac{32}{\delta\epsilon}\Bigr) \hspace{0.2cm}
\end{align}
 pre-momentum steps and momentum steps respectively, with $(1-\frac{1}{n^2})(1-2\delta)(1-\tau^{-\Omega(1)}+e^{-\Omega(d)})$ probability DMStream outputs a vector $q_{K}$ such that
 \begin{equation}
 \sin^2\angle(q_{K},v_1)<\epsilon.
 \end{equation}
\end{theorem}
\begin{proof}
As a reminder, our exact deflation matrix $B$ has spectrum $\lambda_2 > \lambda_3 \geq \lambda_4 \geq \cdots \lambda_n$ with respective eigenvectors $v_2, v_3, \dots, v_n$. We will show that for our choice of $n$ and $J$, that $||H_t||+||G_t||$ for all $t>J$ satisfy the Hardt-Price bounds, therefore allowing our inexact deflation to probabilistically succeed.

First we examine $||H_t||$. By our choice of $n$ and since $d \gg 0$, we have by Proposition \ref{dmspm-bounds} that 
\begin{align}
 ||H_t|| &\leq \frac{\gamma(\lambda_2-\lambda_3)}{5}= \frac{\rho(\lambda_2-\lambda_3)^2}{50\tau\sqrt{d}} < \frac{\rho(\lambda_2-\lambda_3)}{10\tau\sqrt{d}},\\
 ||v_2^\top H_t||&\leq ||H_t|| < \frac{\rho(\lambda_2-\lambda_3)}{10\tau\sqrt{d}}.
\end{align}
Now, we will analyze $||G_t||$. By Proposition $\ref{gt-bound}$, we must consider both $\phi_t$ and $\psi_t$. 

\noindent \textit{Case 1:} $\phi_t > \psi_t$. For our choice of $n$ we have that 
\begin{align}
\phi_t &= |\lambda_1-\lambda_d|\tan^2 \theta(q_0,v_1)\left|\frac{\lambda_2}{\lambda_1}\right|^{2t} + \min\{\frac{\gamma(\lambda_1-\lambda_2)}{5},\frac{\lambda_1-\lambda_2}{5\sqrt{d}}\}\\
&= |\lambda_1-\lambda_d|\tan^2 \theta(q_0,v_1)\left|\frac{\lambda_2}{\lambda_1}\right|^{2t} + 
\frac{\rho(\lambda_2-\lambda_3)}{50\tau\sqrt{d}}\\
&<|\lambda_1-\lambda_d|\tan^2 \theta(q_0,v_1)\left|\frac{\lambda_2}{\lambda_1}\right|^{2t} + 
\frac{\rho(\lambda_2-\lambda_3)}{20\tau\sqrt{d}}.
\end{align}
Solving for $|\lambda_1-\lambda_d|\tan^2 \theta_0\left|\frac{\lambda_2}{\lambda_1}\right|^{2t}< \frac{\rho(\lambda_2-\lambda_3)}{20\tau\sqrt{d}}$ we get
\begin{equation}
 t=\mathcal{O}\Biggl(\frac{1}{\lambda_1-\lambda_2}\log \biggl(\frac{\tan^2 \theta_0\tau\sqrt{d}}{\rho(\lambda_2-\lambda_3)}\biggr)\Biggr),
 \end{equation}
therefore, in this many steps we will have that $\phi_t< \frac{\rho(\lambda_2-\lambda_3)}{10\tau\sqrt{d}}$.

\noindent \textit{Case 2:} $\psi_t \geq \phi_t$. 
If $\min\{1,\max\{\gamma^2,(\lambda_2/\lambda_1)^{t/2}\tan^2\theta_0 \})\}=1$ then $||G_t||\leq \psi_t=0$, i.e., we have 0 noise. So we consider the more interesting case where our minimum is $\max\{\gamma^2,(\lambda_2/\lambda_1)^{t/2}\tan^2\theta_0 \}$. We first note that for any $\alpha \leq 1$, we have that 
\begin{equation}
    \psi_t =\sqrt{2-2\sqrt{1-\alpha}}\leq \alpha.
\end{equation}
If for all $t$ our max is $\gamma^2$, we have that 
\begin{align}
||G_t||\leq \psi_t \leq \gamma^2 \leq  \frac{\rho^2(\lambda_2-\lambda_3)^2}{100\tau^2d}< \frac{\rho(\lambda_2-\lambda_3)}{10\tau\sqrt{d}}.
\end{align}
Otherwise, solving for 
\begin{equation}
    \left|\frac{\lambda_2}{\lambda_1}\right|^{t/2}\tan^2\theta(q_0,v_1)<\frac{\rho(\lambda_2-\lambda_3)}{10\tau\sqrt{d}}
\end{equation}
we get
\begin{equation}
 t=\mathcal{O}\Biggl(\frac{1}{\lambda_1-\lambda_2}\log \biggl(\frac{\tan^2 \theta_0\tau\sqrt{d}}{\rho(\lambda_2-\lambda_3)}\biggr)\Biggr). \end{equation}
Now, for our choice of $n$, after $J_1 := t=\mathcal{O}\Biggl(\frac{1}{\lambda_1-\lambda_2}\log \biggl(\frac{\tan^2 \theta_0\tau\sqrt{d}}{\rho(\lambda_2-\lambda_3)}\biggr)\Biggr)$ steps, we have that for all $t>J_1$,
\begin{equation}
    ||H_t+G_t||\leq ||H_t||+||G_t|| < \frac{\rho(\lambda_2-\lambda_3)}{10\tau\sqrt{d}} + \frac{\rho(\lambda_2-\lambda_3)}{10\tau\sqrt{d}} <
    \frac{\rho(\lambda_2-\lambda_3)}{5}.
\end{equation}
and 
\begin{equation}
    ||v_2(H_t+G_t)||\leq ||H_t||+||G_t|| < \frac{\rho(\lambda_2-\lambda_3)}{10\tau\sqrt{d}} + \frac{\rho(\lambda_2-\lambda_3)}{10\tau\sqrt{d}} <
    \frac{\lambda_2-\lambda_3}{5\tau\sqrt{d}},
\end{equation}
thereby satisfying the Hardt-Price bounds of equation \ref{hardt-price-bounds}, so after a further $J_2 :=\mathcal{O}\bigl(\frac{\lambda_2}{\lambda_2-\lambda_3}\log\frac{d\tau}{\rho} \bigr)$ steps, our pre-momentum phase, with a total of $J=J_1+J_2$ steps, outputs a vector $q_J$ such that $\sin\theta(q_J,v_1)<\rho$ with probability $(1-\frac{1}{n^2})(1-\tau^{-\Omega(1)}+e^{-\Omega(d)})$. Through dual application of Lemma \ref{general_rayleigh_bound} and Proposition $\ref{prop-rhoprecision}$, similar to the proof of Theorem \ref{thm:main-DMPM}, we have that $\beta=\mu_J^2/4\in [\lambda_2^2/4, \lambda_2^2/4)$, therefore, we may proceed to the momentum phase. 

By our assumptions on the variance of our unbiased estimates $\widehat{A}_t$ in relation to $J$, we have by Theorem \ref{momentum-stream-desa} that after a further $K=\frac{\sqrt{\beta}}{\sqrt{\lambda_1-4\beta}}\log\Bigl(\frac{32}{\delta\epsilon}\Bigr)$ steps our entire algorithm outputs a vector $q_{J+K}$ with $\sin^2 \theta(q_{J+K},v_1)<\epsilon$, with probability $(1-2\delta)$. By multiplying the probability of the pre-momentum phase succeeding with the probability of the momentum phase succeeding, our full claim follows. 
\end{proof}

\section{Experimental Data}\label{app:exp_data}

We provide raw experimental data used to generate various plots displayed in this paper (mostly related to iteration complexity) and include comparative wall-time benchmarks for vanilla power method, Power+M, DMPower, and the Lanczos algorithm.

\begin{table}[H]
\centering
\caption{Iterations required for Power+M to converge at various sub-optimal and optimal $\beta$ assignments. $\beta=\lambda_2^2/4=0.2025$ is the optimal momentum coefficient at this setting, where $\lambda_2 = 0.9$.}
\label{table:one}
\resizebox{\textwidth}{!}{
\begin{tabular}{|| c | c c c c c c c ||} 
 \hline
 \multicolumn{8}{|c|}{\textbf{Sub-optimal $\beta$ Selection for Power+M Data} \small{with Loose Eigengaps: \texttt{spec}=[1,0.9,0.8,\dots,0.8]}} \\
 \hline
 Error Threshold ($\epsilon$) & 10e-9 & 10e-8 & 10e-7 & 10e-6 & 10e-5 & 10e-4 & 10e-3 \\
 \hline\hline
 Vanilla Power Method & 81.097 & 70.309 & 59.230 &  48.746 & 37.263 & 26.873 & 16.671 \\ 
 \hline
 Power+M, \small{$\beta = 0.1025$} & 60.463 & 52.565 & 44.420 & 36.745 & 28.272 & 20.561 & 12.859 \\
 \hline
 Power+M, \small{$\beta = 0.2025$} & \textbf{34.986} & \textbf{30.954} & \textbf{26.764} & \textbf{22.879} & \textbf{18.497} & \textbf{14.472} & \textbf{10.205} \\
 \hline
 Power+M, \small{$\beta = 0.3025$} & 43.579 & 38.605 & 32.439 & 27.265 & 20.991 & 16.182 & 10.981 \\
 \hline
 Power+M, \small{$\beta = 0.4025$} & 89.440 & 76.524 & 65.099 & 54.101 & 42.651 & 31.200 & 19.501 \\
 \hline
 Power+M, \small{$\beta = 0.4225$} & 111.972 &  94.846 &  81.729 &  68.635 &  50.530 & 36.275 & 23.144 \\
 \hline
 Power+M, \small{$\beta = 0.4525$} & 179.062 & 155.476 & 130.219 & 106.71 &  81.776 &  58.348 & 32.412 \\
 \hline
\end{tabular}
}
\end{table}

\begin{table}[H]
\centering
\caption{Absolute difference between final approximation of $\lambda_2$ and true value of $\lambda_2$.}
\label{table:two}
\begin{tabular}{|| c | c c c c c c c ||} 
 \hline
 \multicolumn{8}{|c|}{\textbf{Accuracy vs Simultaneous Power Data} \small{with Loose Eigengaps: \texttt{spec}=[1,0.9,0.8,\dots,0.8]}} \\
 \hline
 Error Threshold ($\epsilon$) & 10e-9 & 10e-8 & 10e-7 & 10e-6 & 10e-5 & 10e-4 & 10e-3 \\
 \hline\hline
 Simultaneous PM & 0.1723 & 0.1721 & 0.1684 & 0.1628 & 0.1466 & 0.1107 & 0.1126 \\
 \hline
 DMPower, $\rho = \epsilon$ & \textbf{0.0000} & \textbf{0.0000} & \textbf{0.0003} & \textbf{0.0017} & \textbf{0.0054} & \textbf{0.0370} & 0.0678 \\
 \hline
 DMPower, $\rho = \sqrt{\epsilon}$ & 0.0065 & 0.0145 & 0.0316 & 0.0696 & 0.0672 & 0.0552 & 0.0574 \\
 \hline
 DMPower, $\rho = \sqrt[3]{\epsilon}$ & 0.0570 & 0.0692 &  0.0648 & 0.0545 & 0.0570 & 0.0577 & 0.0505 \\
 \hline
 DMPower, $\rho = \sqrt[4]{\epsilon}$ & 0.0667 & 0.0538 &  0.0537 & 0.0564 & 0.0587 & 0.0529 & \textbf{0.0488} \\
 \hline
\end{tabular}
\end{table}

\begin{table}[H]
\centering
\caption{Iterations for Vanilla PM, Power+M, and DMPower at \texttt{spec}=[1, 0.99, 0.98,\dots, 0.98].}
\label{table:three}
\begin{tabular}{|| c | c c c c c c ||} 
 \hline
 \multicolumn{7}{|c|}{\textbf{Iteration Complexity Data} \small{with $A\in\mathbb{R}^{10\times 10}$}} \\
 \hline
 Error Threshold ($\epsilon$) & 10e-2 & 10e-3 & 10e-4 & 10e-5 & 10e-6 & 10e-7 \\
 \hline\hline
 Vanilla Power Method & \textbf{2.0} & 77.18 & 185.4 & 285.4 & 385.42 & 474.62 \\ 
 \hline
 Power+M, $\beta = \lambda_2^2/4$ & 2.0 & 40.5 & 98.86 & 143.42 & 199.82 & \textbf{231.74} \\
 \hline
 DMPower, $\rho = \epsilon$ & 3.0 & 44.5 & 104.7 & 139.9 & 197.08 & 243.84 \\
 \hline
 DMPower, $\rho = \sqrt{\epsilon}$ & 3.0 & 33.86 & 93.7
 & 153.48 & 192.86 & 249.52 \\
 \hline
 DMPower, $\rho = \sqrt[3]{\epsilon}$ & 3.0 & 36.8 & 102.36 & \textbf{138.82} & 200.48 & 234.48 \\
 \hline
 DMPower, $\rho = \sqrt[4]{\epsilon}$ & 3.0 & 37.7 & \textbf{81.28} & 144.02 & \textbf{192.72} & 243.58 \\
 \hline
 Lanczos Algorithm & 10.0 & \textbf{10.0} & 89.08 & 184.66 & 284.78 & 388.56 \\ 
 \hline

\end{tabular}
\end{table}

\vspace*{-5mm}

\begin{table}[H]
\centering
\label{table:four}
\begin{tabular}{|| c | c c c c c c ||} 
 \hline
 \multicolumn{7}{|c|}{\textbf{Iteration Complexity Data} \small{with $A\in\mathbb{R}^{100\times 100}$}} \\
 \hline
 Error Threshold ($\epsilon$) & 10e-2 & 10e-3 & 10e-4 & 10e-5 & 10e-6 & 10e-7 \\
 \hline\hline
 Vanilla Power Method & \textbf{1.0} & 141.18 & 211.34 & 293.94 & 378.1 & 472.98 \\ 
 \hline
 Power+M, $\beta = \lambda_2^2/4$ & 2.0 & \textbf{68.66} & 120.4 & \textbf{152.32} & 197.84 & 262.8 \\
 \hline
 DMPower, $\rho = \epsilon$ & 3.0 & 76.54 & 113.9 & 156.6 & 203.26 & 259.2 \\
 \hline
 DMPower, $\rho = \sqrt{\epsilon}$ & 3.0 & 71.92 & 116.5 & 156.08 & \textbf{191.64} & 257.66 \\
 \hline
 DMPower, $\rho = \sqrt[3]{\epsilon}$ & 3.0 & 75.44 & 114.6 & 159.08 & 194.34 & \textbf{238.66} \\
 \hline
 DMPower, $\rho = \sqrt[4]{\epsilon}$ & 3.0 & 77.1 & \textbf{107.04} & 158.2 & 196.74 & 245.38 \\
 \hline
 Lanczos Algorithm & 100.0 & 100.0 & 206.06 & 315.62 & 431.12 & 551.76 \\ 
 \hline

\end{tabular}
\end{table}

\vspace*{-5mm}

\begin{table}[H]
\centering
\label{table:five}
\begin{tabular}{|| c | c c c c c c ||} 
 \hline
 \multicolumn{7}{|c|}{\textbf{Iteration Complexity Data} \small{with $A\in\mathbb{R}^{500\times 500}$}} \\
 \hline
 Error Threshold ($\epsilon$) & 10e-2 & 10e-3 & 10e-4 & 10e-5 & 10e-6 & 10e-7 \\
 \hline\hline
 Vanilla Power Method & \textbf{1.0} & 185.98 & 259.22 & 360.12 & 429.16 & 489.4 \\ 
 \hline
 Power+M, $\beta = \lambda_2^2/4$ & 1.0 & 93.16 & 133.48 & 163.18 & 214.64 & 259.0 \\
 \hline
 DMPower, $\rho = \epsilon$ & 2.0 & 102.94 & 133.18 & 163.98 & \textbf{203.88} & 264.72 \\
 \hline
 DMPower, $\rho = \sqrt{\epsilon}$ & 2.0 & \textbf{93.06} & \textbf{131.76} & 171.62 & 207.24 & 252.4 \\
 \hline
 DMPower, $\rho = \sqrt[3]{\epsilon}$ & 2.0 & 97.56 & 143.66 & 170.06 & 212.36 & \textbf{252.24} \\
 \hline
 DMPower, $\rho = \sqrt[4]{\epsilon}$ & 2.0 & 94.12 & 137.5 & \textbf{161.2} & 213.24 & 262.36 \\
 \hline
 Lanczos Algorithm & 500.0 & 500.0 & 607.24 & 717.84 & 833.56 & 964.18 \\ 
 \hline

\end{tabular}
\end{table} 

\begin{table}[H]
\centering
\caption{Percentage of iterations in the \mbox{\textit{momentum phase}} for \texttt{spec}=[1, 0.99, 0.98,\dots, 0.98] with $A\in\mathbb{R}^{100\times 100}$. DMPower across a variety of $\rho$ settings spends the vast majority of its time in the less computationally-intensive momentum phase.}
\label{table:six}
\begin{tabular}{|| c | c c c c ||} 
 \hline
 \multicolumn{5}{|c|}{\textbf{Phases Iteration Complexity Data} \small{with $A\in\mathbb{R}^{100\times 100}$}} \\
 \hline
 Error Threshold ($\epsilon$) & 10e-4 & 10e-5 & 10e-6 & 10e-7 \\
 \hline\hline
 DMPower, $\rho = \epsilon$ & 64.72\% & 88.15\% & 99.06\% & 98.71\% \\
 \hline
 DMPower, $\rho = \sqrt{\epsilon}$ & 99.6\% & 99.49\% & 99.28\%
 & 98.72\% \\
 \hline
 DMPower, $\rho = \sqrt[3]{\epsilon}$ & 99.6\% & 99.49\% & 99.29\% & 98.75\% \\
 \hline

\end{tabular}
\end{table}

\begin{table}[H]
\centering
\caption{Recorded wall-time convergence speeds (in nanoseconds) for vanilla power method, Power+M with optimal $\beta$ assignment, and various settings of DMPower, and the Lanczos algorithm. Performed 1000 calls using random PSD of various sizes with fixed spectrum $\lambda_1=1, \lambda_2=0.99, \lambda_3=0.98$, and remaining eigenvalues set to $0.98$. In several instances DMPower exhibits faster wall-time speeds than Power+M with optimal $\beta$ assignment, consistently outperforms Lanczos, and markedly accelerates the vanilla power method at all error thresholds ($\epsilon$) tighter than 0.1.}
\label{table:seven}
\begin{tabular}{|| c | c c c c c c ||} 
 \hline
 \multicolumn{7}{|c|}{\textbf{Wall-Time Performance Data} \small{with $A\in\mathbb{R}^{10\times 10}$}} \\
 \hline
 Error Threshold ($\epsilon$) & 10e-2 & 10e-3 & 10e-4 & 10e-5 & 10e-6 & 10e-7 \\
 \hline\hline
 Vanilla Power Method & \textbf{67.32} & 3186.6 & 7440.56 & 11497.34 & 15388.04 & 18846.76 \\ 
 \hline
 Power+M, $\beta = \lambda_2^2/4$ & 101.54 & 1941.46 & 4539.84 & \textbf{6140.58} & \textbf{8514.94} & \textbf{9984.74} \\
 \hline
 DMPower, $\rho = \epsilon$ & 251.28 & 2117.2 & 4813.14 & 6436.1 & 8982.4 & 11550.66 \\
 \hline
 DMPower, $\rho = \sqrt{\epsilon}$ & 262.54 & 1696.62 & 4383.94 & 7066.94 & 8869.68 & 11470.54 \\
 \hline
 DMPower, $\rho = \sqrt[3]{\epsilon}$ & 260.84  & 1837.14 & 4732.32 & 6508.14 & 9440.62 & 16254.96 \\
 \hline
 DMPower, $\rho = \sqrt[4]{\epsilon}$ & 256.94 & 1878.86 & \textbf{3830.36} & 6586.34 & 8938.38 & 11294.3 \\
 \hline
 Lanczos Algorithm & 1463.06 & \textbf{1471.86} & 5818.92 & 7836.14 & 11705.58 & 22829.88 \\ 
 \hline

\end{tabular}
\end{table}

\begin{table}[H]
\centering
\label{table:eight}
\begin{tabular}{|| c | c c c c c c ||} 
 \hline
 \multicolumn{7}{|c|}{\textbf{Wall-Time Performance Data} \small{with $A\in\mathbb{R}^{100\times 100}$}} \\
 \hline
 Error Threshold ($\epsilon$) & 10e-2 & 10e-3 & 10e-4 & 10e-5 & 10e-6 & 10e-7 \\
 \hline\hline
 Vanilla Power Method & \textbf{107.24} & 10027.98 & 16213.82 & 21231.46 & 29691.06 & 31801.24 \\ 
 \hline
 Power+M, $\beta = \lambda_2^2/4$ & 132.2 & \textbf{4890.36} & 9701.76 & \textbf{10840.92} & \textbf{13930.12} & 21100.88 \\
 \hline
 DMPower, $\rho = \epsilon$ & 432.06 & 5969.96 & 8910.08 & 11985.2 & 15986.18 & \textbf{19431.32} \\
 \hline
 DMPower, $\rho = \sqrt{\epsilon}$ & 457.82 & 5765.98 & 9028.7 & 12397.06 & 15048.34 & 19432.9 \\
 \hline
 DMPower, $\rho = \sqrt[3]{\epsilon}$ & 435.48 & 6113.44 & 9393.82 & 12575.0 & 15028.84 & 20037.1 \\
 \hline
 DMPower, $\rho = \sqrt[4]{\epsilon}$ & 432.32 & 6094.92 & \textbf{8143.68} & 12136.16 & 16063.48 & 20166.24 \\
 \hline
 Lanczos Algorithm & 7724.94 & 6672.18 & 13832.8 & 21457.2 & 29447.52 & 39152.46 \\ 
 \hline

\end{tabular}
\end{table}

\begin{table}[H]
\centering
\label{table:nine}
\begin{tabular}{|| c | c c c c c c ||} 
 \hline
 \multicolumn{7}{|c|}{\textbf{Wall-Time Performance Data} \small{with $A\in\mathbb{R}^{500\times 500}$}} \\
 \hline
 Error Threshold ($\epsilon$) & 10e-2 & 10e-3 & 10e-4 & 10e-5 & 10e-6 & 10e-7 \\
 \hline\hline
 Vanilla Power Method & \textbf{307.9} & 27496.5 & 35858.14 & 49056.32 & 66559.52 & 70891.42 \\ 
 \hline
 Power+M, $\beta = \lambda_2^2/4$ & 337.26 & \textbf{12912.76} & \textbf{20098.66} & 24924.28 & 31590.88 & 49575.94 \\
 \hline
 DMPower, $\rho = \epsilon$ & 1423.64 & 15366.22 & 20172.32 & \textbf{24663.36} & \textbf{30818.32} & 38748.8 \\
 \hline
 DMPower, $\rho = \sqrt{\epsilon}$ & 2094.86 & 15868.62 & 20382.9 & 27008.76 & 30883.86 & 38207.12 \\
 \hline
 DMPower, $\rho = \sqrt[3]{\epsilon}$ & 1490.0 & 15209.92 & 22283.78 & 26917.66 & 31475.8 & \textbf{36617.02} \\
 \hline
 DMPower, $\rho = \sqrt[4]{\epsilon}$ & 1623.94 & 17126.1 & 20743.88 & 25229.04 & 32106.58 & 39925.98 \\
 \hline
 Lanczos Algorithm & 58438.68 & 58700.02 & 80682.24 & 91623.18 & 106078.32 & 126685.16 \\ 
 \hline

\end{tabular}
\end{table}

\begin{table}[H]
\centering
\caption{Averaged $\log$ error, batch size = 500, with performance measured by $\log_{10}\bigl(1 - \frac{||X^\top q_K||}{||X^\top v_1||} \bigr)$. DMStream registers much better accuracy than Oja's algorithm and emulates the performance of optimal Mini-Batch Power+M. We notice that accuracy does not improve as we increase the number of epochs, which is commonly observed for streaming algorithm running with small batch sizes. We demonstrate in Figure \ref{fig:stream} that increasing batch size results in improved accuracy.}
\label{table:ten}
\begin{tabular}{|| c | c c c c c ||} 
 \hline
 \multicolumn{6}{|c|}{\textbf{Log Error Performance Data}  with batch size = 500} \\
 \hline
 Epochs & 10 & 20 & 30 & 40 & 50 \\
 \hline\hline
 DMStream, $\rho = 0.1$ & -1.900 & -1.894 & -1.983 & -1.969 & -1.959 \\
 \hline
 DMStream, $\rho = 0.01$ & -1.992 & -1.908 & -1.882 & -1.949 & -1.905 \\
 \hline
 DMSteam, $\rho = 0.001$ &   -1.929 & -1.9585 & -1.936 & -1.963 & -1.973 \\
 \hline
 Oja $\eta_t=3/t$ & -0.588 & -0.599 & -0.625 & -0.565 & -0.549 \\
 \hline
 Oja $\eta_t=9/t$ &  -0.629 & -0.592 & -0.599 & -0.531 & -0.638 \\
 \hline
 Oja $\eta_t=27/t$ &  -0.680 & -0.668 & -0.584 & -0.599 & -0.647 \\
 \hline
 Oja $\eta_t=81/t$ & -0.590 & -0.676 & -0.527 & -0.604 & -0.665 \\
 \hline
 Mini-Batch Power+M (optimal $\beta=\lambda_2^2/4$) &  -1.881 & -1.860 & -1.964 & -1.996 & -1.966 \\
 \hline
 
\end{tabular}
\end{table}
\section{Precision Bounds for Momentum}\label{app:precision_bounds}
In our main result Theorem $\ref{thm:main-DMPM}$, we assume that $\beta \in [\lambda_2^2/4,\lambda_1^2/4)$, while it is possible that we select $\beta<{\lambda_2^2}/4$. This poses no problem: as long as $\beta$ is near $\lambda_2^2/4$ we will experience similar acceleration effects. We state the more general version of Theorem \ref{PowerMThm} which reflects this fact and add our own modified condition, $\Delta_{2,3}:=\lambda_2-\lambda_3>0$.
\begin{theorem} [Generalized Convergence of Power+M~\citep{de2018accelerated}]
\label{PowerMThm-Restated}
Given a PSD matrix $A\in\mathbb{R}^{d\times d}$ with eigenvalues $ \lambda_1>\lambda_2 > \lambda_3\dots \lambda_d\geq 0$ with associated orthonormal eigenvectors $v_1,v_2,\dots,v_d$, for a unit $q_0\in\mathbb{R}^{n}$ non-orthogonal to $v_1$, running Power+M with $\beta\leq \lambda_1$ results in $q_k$ with

\begin{equation}
\sin^2(\theta_k)=1-(q_k^{\top}v_1)^2\leq \frac{1}{|q_0^{\top}v_1|^2}\cdot\begin{cases} 
      {4\biggl(\frac{2\sqrt{\beta}}{\lambda_1+\sqrt{\lambda_1^2-4\beta}}\biggr)}^{2k}, & \lambda_2 < 2\sqrt{\beta} \\
      \biggl(\frac{\lambda_2+\sqrt{\lambda_2^2-4\beta}}{\lambda_1+\sqrt{\lambda_1^2-4\beta}}\biggr)^{2k} ,& \lambda_2 \geq 2\sqrt{\beta}
      \end{cases}
\end{equation}
\end{theorem}
Ultimately then, our $\beta$ can land on either side of $\frac{\lambda_2^2}{4}$ and we we will experience acceleration as long as we are "close." The next result establishes tolerance for our estimations $\widehat{\lambda}_2$ of $\lambda_2$.
\begin{proposition}[Proposition \ref{prop-rhoprecision} re-stated]
\label{prop-rhoprecision-restated}
The momentum phase of DMPower set with momentum coefficient $\beta=\widehat{\lambda}_2^2/4=\mu_J^2/4$ converges if and only if
\begin{equation}
 |\lambda_2-\widehat{\lambda}_2| \leq \Delta_{1,2}.
\end{equation}
\end{proposition}
\begin{proof}
If $|\lambda_2-\widehat{\lambda}_2|=|\lambda_2-\mu_J|\leq \rho=\Delta_{1,2}=|\lambda_1-\lambda_2|$, then
\begin{align}
    |\lambda_2-\mu_J|\cdot|\lambda_2+\mu_J|&<|\lambda_1-\lambda_2|\cdot |\lambda_1+\lambda_2|\\
    &\Rightarrow |\lambda_1^2-\mu_J^2|<|\lambda_1^2-\lambda_2^2|\\
    &\Rightarrow
    \frac{1}{4}|\lambda_1^2-\mu_J^2|<\frac{1}{4}|\lambda_1^2-\lambda_2^2|
\end{align}
The first line of our inequality follows from the fact that $\mu_J>0$ since $A$ is PSD, therefore implicitly we have that $0<\mu_J<\lambda_1$, implying $|\lambda_2+\mu_J|<|\lambda_2+\lambda_1|$. This shows we satisfy the constraint. To show that this is necessary and sufficient, we consider the case where $\rho>\Delta_{1,2}$. This allows for possible selection of $\mu_J>\lambda_1$, in which case we have that $\frac{1}{4}|{\lambda_2}^2-\mu_k^2|>\frac{1}{4}|\lambda_1^2-\lambda_2^2|$, which is outside of our guaranteed interval for convergence.
\end{proof}

\section{Data Matrix Generation for Non-Streaming Experiments}\label{app:data_matrix}
For every experiment, we ran variations of the power method on a random covariance matrix $A$ with a fixed spectrum. We constructed such matrices using a synthetic singular value decomposition (SVD). Specifically, we begin with a diagonal $d\times d$ matrix $\Sigma=\textrm{diag}\{1,\sqrt{\lambda_2},\sqrt{\lambda_3},\sqrt{\lambda_4},\dots,\sqrt{\lambda_d}\}$. Notice that by default we set $\lambda_1=1$. In practice, we set $\sqrt{\lambda_3}=\sqrt{\lambda_4}=\cdots=\sqrt{\lambda_d}$ for simplicity. We then drew two random orthogonal matrices $U\in\mathbb{R}^{1000 \times d}$ and $V\in\mathbb{R}^{d\times d}$ from the Haar distribution. We then form the data matrix $X=dU\Sigma V^{\top}\in\mathbb{R}^{1000\times d}$, from which we acquire our covariance matrix $A=\frac{1}{1000}XX^{\top}$ with spectrum $\textrm{diag}\{\lambda_1,\lambda_2,\dots,\lambda_d\}$. For every run in our non-streaming experiments, we would generate a \textrm{new} covariance matrix with our desired spectrum using this construction. 
\section{Spectral Clustering \& Experimental Data}
\label{app:spectralcluster}
We first describe provide and discuss the algorithm used for deflation-based power iteration clustering \mbox{\citep{thang2013deflation}}. We then provide the results of our spectral clustering experiments. 
\protect{\begin{algorithm}
\caption{Deflation-based Power Iteration Clustering (DPIC)}
\label{alg:piccluster}
\begin{algorithmic}[1]
\Require normalized affinity matrix $W\in\mathbb{R}^{d\times d}$
\State $W_0=W$
\For{$i=1,2$} \Comment{Top-2 component recovery}
\State $v_i = \textrm{PowerIteration}(W_{i-1})$ \Comment{Return the leading eigenvector}
\State $W_i = W_{i-1}-\frac{W_{i-1}v_iv_i^{\top}W_{i-1}}{v_i^{\top}W_{i-1}v_i}$ \Comment{Schur complement deflation}
\State $i \leftarrow i+1$
\EndFor
\State $\textrm{Use k-means on eigenvector approximates }v_1,v_2.$\\
\Return $C_1, C_2$ 
\end{algorithmic}
\end{algorithm}}
We note that PowerIteration in step 3 of Algorithm \mbox{\ref{alg:piccluster}} may be replaced with DMPower, Power+M, or any other variant. To recover the second eigenvector, DPIC uses a Schur complement deflation \mbox{\citep{saad2011numerical}} on $W$ once the leading eigenvector is computed. Briefly, deflating a matrix shifts its spectrum so that the second leading eigenvalue/eigenvector now becomes the leading eigenvalue/eigenvector, upon which a power iteration may be used again. Successive deflations allows one to recover as many eigenvectors as desired, although numerical instability is increased with each deflation. 

\begin{table}[H]
\centering
\caption{Proportion of data points correctly classified by spectral clustering combined with k-means. As the error-threshold (and therefore, accuracy) of the eigenvector output is tightened, accuracy improves, eventually achieving perfect classification. See Figure \mbox{\ref{fig:cluster}} for a visual depiction of data separation using DMPower ($\rho=\sqrt[3]{\epsilon})$. We further note that although the vanilla power method achieves perfect classification at lower error-thresholds, it is at the cost of significantly more iterations as indicated in Table \mbox{\ref{table:thirteen}}.}
\label{table:eleven}
\begin{tabular}{|| c | c c c c c ||} 
 \hline
 \multicolumn{6}{|c|}{\textbf{Spectral Clustering Accuracy Data} on Concentric Circles Dataset} \\
 \hline
 Error Threshold ($\epsilon$) & 10e-2 & 10e-4 & 10e-6 & 10e-8 & 10e-10 \\
 \hline\hline
 Vanilla Power Method & 0.6953 & 0.7854 & 1.0000 & 1.0000 & 1.0000 \\
 \hline
 Power+M, $\beta = \lambda_2^2/4$ & 0.6679 & 0.7158 & 0.7810 & 0.9886 & 1.0000 \\
 \hline
 DMPower, $\rho = \epsilon$ & 0.6997 & 0.7215 & 0.8191 & 0.9869 & 1.0000 \\
 \hline
 DMPower, $\rho = \sqrt{\epsilon}$ & 0.6908 & 0.6927 & 0.7779 & 0.9881 & 1.0000 \\
 \hline
 DMPower, $\rho = \sqrt[3]{\epsilon}$ & 0.6770 & 0.7056 & 0.7657 & 0.9872 & 1.0000 \\
 \hline
 
\end{tabular}
\end{table}

\vspace*{-5mm}

\begin{table}[H]
\centering
\label{table:twelve}
\begin{tabular}{|| c | c c c c c ||} 
 \hline
 \multicolumn{6}{|c|}{\textbf{Spectral Clustering Accuracy Data} on Half Moons Dataset} \\
 \hline
 Error Threshold ($\epsilon$) & 10e-2 & 10e-4 & 10e-6 & 10e-8 & 10e-10 \\
 \hline\hline
 Vanilla Power Method & 0.6164 & 0.7793 & 0.9808 & 1.0000 & 1.0000 \\
 \hline
 Power+M, $\beta = \lambda_2^2/4$ & 0.5966 & 0.6362 & 0.7560 & 1.0000 & 1.0000 \\
 \hline
 DMPower, $\rho = \epsilon$ & 0.6196 & 0.6616 & 0.8185 & 0.9696 & 1.0000 \\
 \hline
 DMPower, $\rho = \sqrt{\epsilon}$ & 0.6132 & 0.6112 & 0.7368 & 1.0000 & 1.0000 \\
 \hline
 DMPower, $\rho = \sqrt[3]{\epsilon}$ & 0.6084 & 0.6283 & 0.8054 & 1.0000 & 1.0000 \\
 \hline
 
\end{tabular}
\end{table}

\vspace*{-5mm}

\begin{table}[H]
\centering
\caption{Iteration complexity required to recover principal components. DMPower requires significantly fewer iterations for eigenvector recovery when compared to the vanilla power method, and closely mimics the performance of Power+M with $\beta=\lambda_2^2/4$.}
\label{table:thirteen}
\begin{tabular}{|| c | c c c c c ||} 
 \hline
 \multicolumn{6}{|c|}{\textbf{Spectral Clustering Iterations Data} on Concentric Circles Dataset} \\
 \hline
 Error Threshold ($\epsilon$) & 10e-2 & 10e-4 & 10e-6 & 10e-8 & 10e-10 \\
 \hline\hline
 Vanilla Power Method & 7.72 & 59.16 & 605.36 & 1457.56 & 2426.36 \\
 \hline
 Power+M, $\beta = \lambda_2^2/4$ & \textbf{3.00} & \textbf{9.24} & 85.12 & \textbf{642.76} & \textbf{1321.00} \\
 \hline
 DMPower, $\rho = \epsilon$ & 6.00 & 11.16 & 92.00 & 751.68 & 1449.08 \\
 \hline
 DMPower, $\rho = \sqrt{\epsilon}$ & 5.00 & 11.20 & \textbf{77.68} & 694.20 & 1452.84 \\
 \hline
 DMPower, $\rho = \sqrt[3]{\epsilon}$ & 5.00 & 10.64 & 78.48 & 642.92 & 1505.88 \\
 \hline
 
\end{tabular}
\end{table}

\vspace*{-5mm}

\begin{table}[H]
\centering
\label{table:fourteen}
\begin{tabular}{|| c | c c c c c ||} 
 \hline
 \multicolumn{6}{|c|}{\textbf{Spectral Clustering Iterations Data} on Half Moons Dataset} \\
 \hline
 Error Threshold ($\epsilon$) & 10e-2 & 10e-4 & 10e-6 & 10e-8 & 10e-10 \\
 \hline\hline
 Vanilla Power Method & 7.12 & 58.72 & 506.00 & 1061.68 & 1929.52 \\
 \hline
 Power+M, $\beta = \lambda_2^2/4$ & \textbf{3.00} & \textbf{11.20} & \textbf{84.84} & \textbf{601.12} & 1226.72 \\
 \hline
 DMPower, $\rho = \epsilon$ & 6.00 & 12.08 & 89.16 & 629.96 & 1192.68 \\
 \hline
 DMPower, $\rho = \sqrt{\epsilon}$ & 5.00 & 11.56 & 94.88 & 661.20 & 1225.08 \\
 \hline
 DMPower, $\rho = \sqrt[3]{\epsilon}$ & 5.00 & 11.52 & 92.64 & 711.00 & \textbf{1180.72} \\
 \hline
 
\end{tabular}
\end{table}

\end{document}